\documentclass[10pt,journal,compsoc]{IEEEtran}

\ifCLASSOPTIONcompsoc
  \usepackage[nocompress]{cite}
\else
  \usepackage{cite}
\fi

\usepackage{bm}
\usepackage{graphicx}  
 \usepackage{mathrsfs}

\usepackage{cite}  

\usepackage[colorlinks=true]{hyperref}  

\usepackage{amssymb}
\usepackage{algorithm}
\usepackage{amsmath}
\usepackage{amsthm}
\usepackage{tabularx,multirow,booktabs}
\usepackage{balance}
\usepackage[noend]{algpseudocode}
\usepackage{url}
\usepackage{epsfig,subfigure}

 \newtheorem{definition}{\bf{Definition}}
 \newtheorem{theorem}{\bf{Theorem}}

\newcommand{\Xcal}{{\mathcal{X}}}

\newcommand{\lrincir}[1]{\left( #1 \right)}

\newcommand{\lrnorm}[1]{\left\lVert#1\right\rVert}
\newcommand{\lrangle}[1]{\left\langle#1 \right\rangle}

\newcommand{\RR}{\mathbb{R}}

\begin{document}

\title{Triangle Lasso for Simultaneous Clustering and Optimization in Graph Datasets}

\author{Yawei Zhao,
        Kai Xu,
        En Zhu$^{\ast}$,
        Xinwang Liu,
        Xinzhong Zhu,
        Jianping Yin$^{\ast}$
        \thanks{${\ast}$ represents the corresponding authors.}
\IEEEcompsocitemizethanks{
\IEEEcompsocthanksitem Yawei Zhao, Kai Xu, Xinwang Liu and En Zhu are with College of Computer, National University of Defense Technology, Changsha, Hunan, 410073, China. E-mail: zhaoyawei@nudt.edu.cn; kevin.kai.xu@gmail.com; xinwangliu@nudt.edu.cn; enzhu@nudt.edu.cn.
\IEEEcompsocthanksitem  Xinzhong Zhu is with School of Life Science and Technology, XIDIAN University, Xi'an, Shanxi, China, 710126 and  College of Mathematics, Physics and Information Engineering, Zhejiang Normal University, Jinhua, Zhejiang, 321004, China. He makes equal contribution with Yawei Zhao. E-mail: zxz@zjnu.edu.cn.
\IEEEcompsocthanksitem  Jianping Yin is with Dongguan University of Technology, Dongguan, Guangdong, 523000, China. E-mail: jpyin@dgut.edu.cn.
} 
}


\IEEEcompsoctitleabstractindextext{%
\begin{abstract}
Recently, network lasso has drawn many attentions due to its remarkable performance on simultaneous clustering and optimization. However,  it usually suffers from the imperfect data (noise, missing values etc), and yields sub-optimal solutions. The reason is that it finds the similar instances according to their features directly, which is usually impacted by the imperfect data, and thus returns sub-optimal results.    In this paper, we propose  triangle lasso to avoid its disadvantage for graph datasets. In a graph dataset, each instance is represented by a vertex. If two instances have many common adjacent vertices, they tend to become similar. Although some instances are profiled by the imperfect data, it is still able to find the similar counterparts. Furthermore, we develop an efficient algorithm based on Alternating Direction Method of Multipliers (ADMM) to obtain a moderately accurate solution. In addition, we present a dual method to obtain the accurate solution with the low additional time consumption. We demonstrate through extensive  numerical experiments that triangle lasso is robust to the imperfect data. It usually yields a better performance than the state-of-the-art method when performing data analysis tasks in practical scenarios.  
\end{abstract}

}
\maketitle

\IEEEpeerreviewmaketitle

\section{Introduction}
\label{introduction}
\IEEEPARstart{I}{t} has been attractive to  find  the similarity among instances and conduct data analysis simultaneously via convex optimization for recent years. Let us take an example to explain this kind of tasks. Consider the price prediction of houses in New York. Suppose that we use ridge regression to conduct prediction tasks. We need to learn the weights of features (cost,  area, number of rooms etc) for each house. The price of the houses situated in a district should be predicted by using the similar or identical weights due to the same location-based factors, e.g. school district etc. But, those location-based factors are usually difficult to be quantified as the additional features. Thus, it is challenging to predict the price of houses under those location-based factors.   Recently,  network lasso is proposed to conduct this kind of tasks, and yields remarkable performance \cite{Hallac:2015fy}.

However, it is worth noting that some features of an instance are usually missing, noisy or unreliable in practical scenarios, which are collectively referred to as the \textit{imperfect} data in the paper. For instance, the true cost of a house is usually a core secret for a company, which cannot be obtained in many cases. The market expectation is usually not stable, and has some fluctuations for a period of time. Network lasso suffers from the imperfect data (noise, missing values etc), and yields sub-optimal solutions. One of the reasons is that they use those features directly to learn the unknown weights whose accuracy is usually impaired due to such  imperfect data. It is thus challenging to learn the correct weights for a house, and make a precise prediction. Therefore,  it is valuable to develop a robust method to handle the imperfect data.

  Many excellent  researches have been conducted  and obtain impressive results.  There are some pioneering researches in convex clustering \cite{Wang:2016kw,Wangsparse}.  \cite{Wang:2016kw} focuses on finding and removing the outlier features. \cite{Wangsparse}  is proposed to find and remove the uninformative features in a high dimensional clustering scenario. Assuming that those targeting features are sparse,  the pioneering researches successfully find and remove them via an $l_{2,1}$ regularization.  However, their methods rely on an extra hyper-parameter for such a regularized item. The extra need-to-tune hyper-parameter limits their usefulness in practical tasks.  As an extension of convex clustering, network lasso is proposed to conduct clustering and optimization simultaneously for large graphs \cite{Hallac:2015fy}.  It formulates the empirical loss of each vertex into the loss function and  each edge into the regularization. If the imperfect data exists, the formulations of the vertex and edge are inaccurate. Due to such inaccuracy, network lasso returns  sub-optimal  solutions. As the pioneering researches, \cite{Jung:2017ug} investigates the conditions on the network topology to obtain the accurate clustering result in the network lasso.  However, given a network topology, it is still not able to handle vertices with the imperfect data. Additionally, we find that it is not efficient for those previous methods, which impedes them to be used in the practical scenarios. In a nutshell, it is important to propose a robust method to handle those imperfect data and meanwhile yield the solution efficiently.

In the paper, we introduce \textit{triangle lasso} to conduct data analysis and clustering simultaneously via convex optimization. Triangle lasso re-organizes a dataset as a graph or network\footnote{The graph and network have the  equivalent meanings in the paper.}. Each instance is represented by a vertex. If two instances are closely related in a data analysis task, they are connected by an edge. Here, the \textit{related} has various meanings for specific tasks. For example, two vertices may be connected if  an vertex is one of the $k$ nearest neighbours of the other one. Our key idea is illustrated in Fig. \ref{figure_illustration_triangle_lasso2}.  If there is a shared neighbour between a vertex and its direct adjacency, a triangle exists. It implies that the vertices may be  similar. If two vertices exist in multiple triangles, they tend to be very similar because that they have many shared neighbours. Benefiting from the triangles, triangle lasso is robust to the imperfect values. Although a vertex, e.g. $v_i$ has some noisy values, we can still find its similar counterpart $v_j$ and $v_k$ via their shared neighbours. 

\begin{figure}[t]
\centering 
\includegraphics[width=0.75\columnwidth]{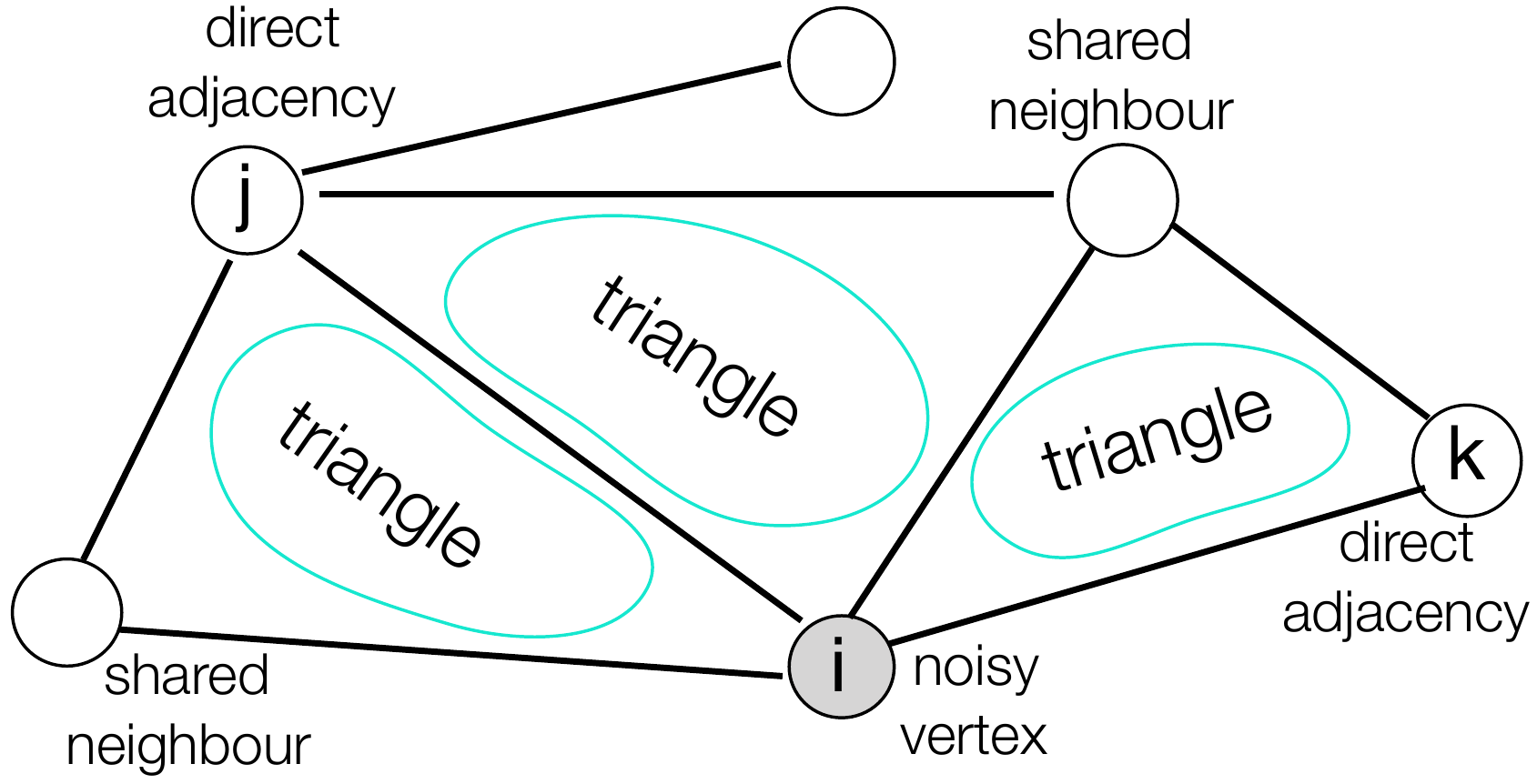}
\caption{The noisy vertex $v_i$ is more similar to $v_j$ than to $v_k$ because there are two shared neighbours between them. }
\label{figure_illustration_triangle_lasso2}
\end{figure}

It is worthy noting that the neighbouring information of a vertex is formulated into a sum-of-norms regularization in  triangle lasso. It is challenging to solve the triangle lasso efficiently due to three reasons. First, it is non-separable for the weights of the adjacent vertices. If a vertex has a large number of neighbours, it is time-consuming to obtain the optimal weights.  Second, the objective function is non-smooth at the optimum when more than one vertex belongs to a cluster. In triangle lasso, if two vertices belong to a cluster, their weights are identical. But, the sum-of-norms regularization implies that the objective function is non-differentiable in the case. There usually exist a large number of non-smooth points for a specific task. Third, we have to optimize a large number of variables, i.e. $O(nd)$, where $n$ is the number of instances, and $d$ is the number of features. In the paper,  we develop an efficient method based on ADMM to obtain a moderately accurate solution. After that, we transform  the triangle lasso to an easy-to-solve Second Order Cone Programming (SOCP) problem in the dual space. Then, we propose a dual method to obtain the accurate solution efficiently. Finally, we use the learned weights to conduct various data analysis tasks. Our contributions are outlined as follows:
\begin{itemize}
\item We formulate the triangle lasso as a general robust optimization framework.
\item  We provide an ADMM method to yield the moderately accurate solution, and develop a dual method to obtain the accurate solution efficiently.
\item  We demonstrate that triangle lasso is robust to the imperfect data, and yields the solution efficiently  according to empirical studies.
\end{itemize}

The rest of paper is organized as follows. Section \ref{sect_related_work} outlines the related work. Section \ref{sect_formulation} presents the formulation of triangle lasso. Section \ref{sect_admm} presents our ADMM method which obtains a moderately accurate solution. Section \ref{sect_dual} presents the dual method which obtains the accurate solution. Section \ref{sect_time_complexity} discusses the time complexity  of our proposed methods. Section \ref{sect_empirical_studies} presents the evaluation studies. Section  \ref{sect_conclusion} concludes the paper.

\section{Related Work}
\label{sect_related_work}
Recently, there are a lot of excellent researches on clustering and data analysis simultaneously, and they obtain impressive results. 

\subsection{Convex clustering}
As a specific field of triangle lasso, convex clustering has drawn many attentions \cite{icml2017,Chi:2013ey,Han:2016tg,Tan:2015vr,Chi:2016bu,Wang:2016kw, Wangsparse}.  \cite{icml2017} proposes a new stochastic incremental algorithm to conduct convex clustering. \cite{Chi:2013ey} proposes a splitting method to conduct convex clustering via ADMM. \cite{Han:2016tg} proposes a reduction technique to conduct graph-based convex clustering. \cite{Tan:2015vr} investigates the statistical properties of convex clustering. \cite{Chi:2016bu} formulates a new variant of convex clustering, which conducts clustering on instances and features simultaneously. Comparing with our methods, those previous researches focus on improving the efficiency of convex clustering, which cannot handle the imperfect data.   \cite{Wang:2016kw} uses an $l_{2,1}$ regularization to pick noisy features when conducting convex clustering. \cite{Wangsparse} investigates to remove sparse outlier or uninformative features when conducting convex clustering. However, both of them uses more than one convex regularized items in the formulation, which needs to tune multiple hyper-parameters in practical scenarios. Specifically,  the previous methods including \cite{Wang:2016kw,Wangsparse} focus on being robust to the imperfect data. They usually add a new regularized item, e.g. $l_1$-norm regularization or $l_{2,1}$-norm regularization to obtain a sparse solution. Although it is effective, the newly-added regularized item usually need to tune a hyper-parameter for the regularized item, which limits their use in the practical scenarios. 
\subsection{Network lasso}
 As the extension of convex clustering, network lasso is good at conducting clustering and optimization simultaneously \cite{Hallac:2015fy,Ghosh2016An,Jung:2017ug}. As a general framework, network lasso yields remarkable performance in various machine learning tasks \cite{Hallac:2015fy,Ghosh2016An}.  However, its solution is easily impacted by the imperfect data, and yields sub-optimal solutions in the practical tasks. \cite{Jung:2017ug} investigates the network topology in order to obtain accurate solution.  Triangle lasso aims to obtain a robust solution with inaccurate vertices for a known network topology, which is orthogonal to \cite{Jung:2017ug}.

\begin{figure}[t]
\centering 
\subfigure[$6$ clusters, $\alpha=4$]{\includegraphics[width=0.48\columnwidth]{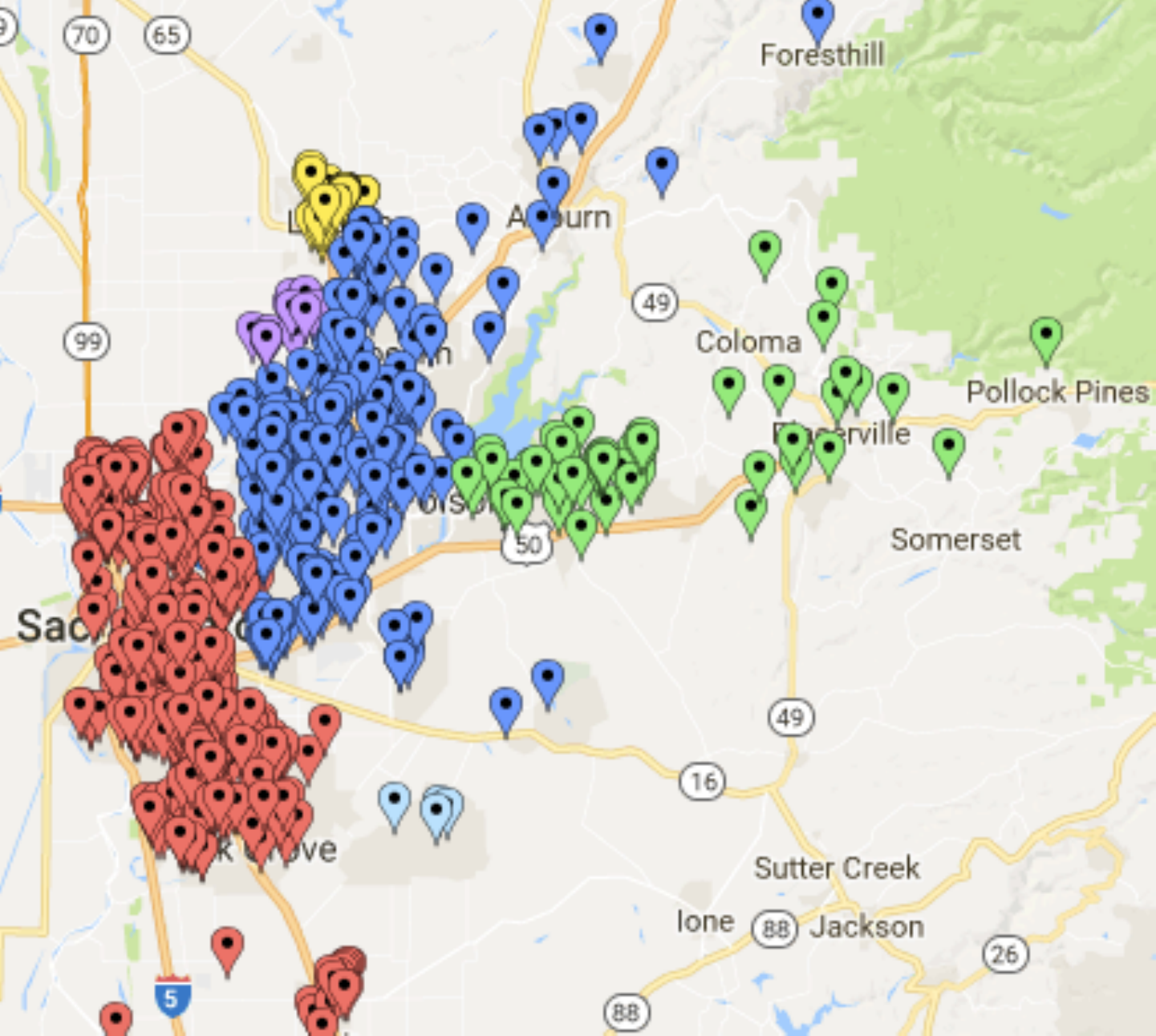}\label{figure_example_triangle_6clusters}}
\hspace{3pt}
\subfigure[$4$ clusters, $\alpha=5$]{\includegraphics[width=0.475\columnwidth]{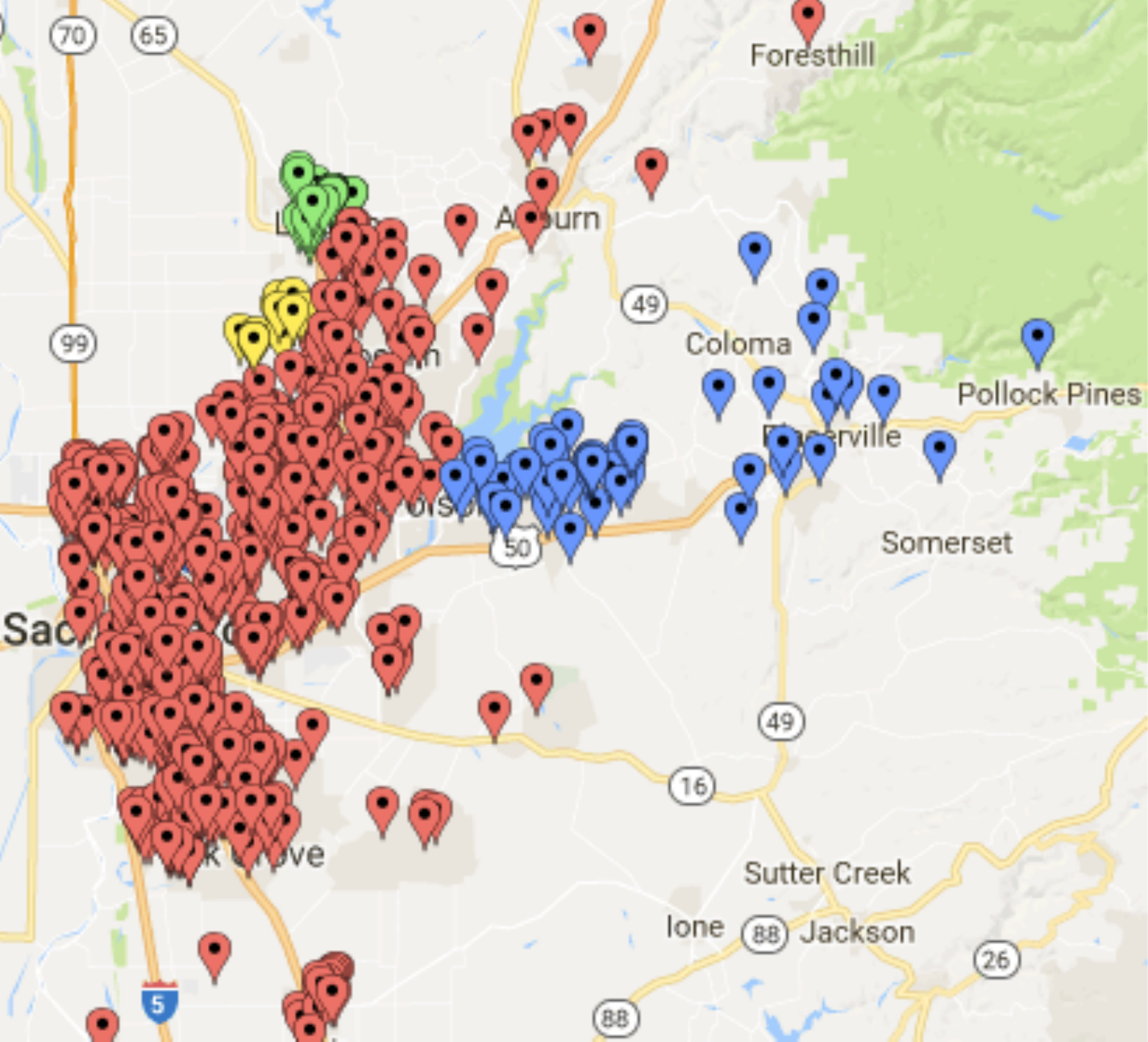}\label{figure_example_triangle_4clusters}}
\caption{The prediction of the house price in the Greater Sacramento area.  With the increase of $\alpha$, more houses are fused to a cluster. The houses located in a cluster use an identical weight to predict their prices. }
\label{figure_example_house}
\end{figure}

\section{Problem formulation}
\label{sect_formulation}
In this section, we first present the formulation of the triangle lasso. Then, we instantiate it in some applications, and present the result in a demo example. After that, we present the workflow of the triangle lasso. Finally, we shows the symbols used in the paper and their notations.

\subsection{Formulation}
We formulate the triangle lasso as an unconstrained convex optimization problem:
\begin{eqnarray}
\setlength{\abovedisplayskip}{0pt}
\setlength{\belowdisplayskip}{0pt}
\nonumber
\min_{\tiny X\in \mathbb{R}^{n\times d}} \sum\limits_{v_i\in \mathcal{V}} f_{i}(X_i, y_i) + \alpha \sum\limits_{e_{ij}\in \mathcal{E}}w_{ij}g_{i,j}(X_i,X_j). 
\end{eqnarray}
Given a graph, $\mathcal{V}$ represents the vertex set containing $n$ vertices, and $\mathcal{E}$ represents the edge set containing $m$ edges. $y_i$ represents the response of the $i$-th instance, i.e. $A_i$. $w_{ij}$ denotes the weight for the edge $e_{ij}$, which could be specifically defined according to the task in practical.  $\alpha>0$ is the regularization  coefficient. It is highlighted that $f_{i}(X_i, y_i)$ becomes $f_{i}(X_i)$ in the unsuperivised learning tasks such as clustering because that there is no response for an instance in the unsupervised learning tasks.  We let
\begin{eqnarray}
\setlength{\abovedisplayskip}{0pt}
\setlength{\belowdisplayskip}{0pt}
\nonumber
g_{i,j}(X_i,X_j) \mathrm{=} g_0(X_i,X_j) \mathrm{+}g_1(X_i,N(X_j)) \mathrm{+}g_2(X_j,N(X_i))
\end{eqnarray} hold, where $N(\cdot)$ represents the neighbour set of a vertex. $f_{i}(X_i,y_i)$ represents the empirical loss on the vertex $v_i$.  As a regularization,  $g_0$, $g_1$ and $g_2$ can have various formulations. In the paper, we focus on the sum-of-norms regularization, i.e.,
\begin{eqnarray}
\setlength{\abovedisplayskip}{0pt}
\setlength{\belowdisplayskip}{0pt}
\nonumber
&&g_{i,j}(X_i,X_j) \\ \nonumber
&\mathrm{=}& \lVert X_i \mathrm{-} X_j \rVert \mathrm{+} \sum\limits_{\tiny \substack{k_j \in N(X_j); \\(i, k_j)\in \mathcal{E}}} \lVert  X_i \mathrm{-} X_{k_j} \rVert \mathrm{+} \sum\limits_{\tiny \substack{k_i\in N(X_i);\\ (j, k_i)\in \mathcal{E}}} \lVert  X_j \mathrm{-} X_{k_i} \rVert.
\end{eqnarray} 

Since $g_{i,j}$ is the sum of $l_2$ norms, i.e. $l_1/l_2$ norm, we denote it by $(1,2)$-norm. Given $n$ vertices and $m$ edges, define an auxiliary matrix $Q \in \mathbb{R}^{m\times n}$.  The non-zero elements of a row of $Q$ represent an edge. If the edge is $e_{ij}$, and it is represented by the $k$-th row of $Q$, we have
\begin{eqnarray}
\setlength{\abovedisplayskip}{0pt}
\setlength{\belowdisplayskip}{0pt}
\nonumber
Q_k=\begin{pmatrix}
\mathbf{0}_{1\times (i-1)},\alpha w_{ij} q_{ij},\mathbf{0}_{1\times (j-i-1)},-\alpha w_{ij} q_{ij}, \mathbf{0}_{1\times n-j}
\end{pmatrix}
\end{eqnarray} where $q_{ij}$ is a positive known integer for a known graph. Triangle lasso is finally formulated as:
\begin{eqnarray}
\setlength{\abovedisplayskip}{0pt}
\setlength{\belowdisplayskip}{0pt}
\label{equa_triangle_lasso}
\min_{\tiny X\in \mathbb{R}^{n\times d}} \sum\limits_{i=1}^n f_{i}(X_i, y_i) + \lVert QX\rVert_{1,2}.
\end{eqnarray}

Note that $\alpha$  is a hyper-parameter for triangle lasso, and it can be varied in order to control the similarity  between instances. 
Additionally, the global minimum of ($\ref{equa_triangle_lasso}$) is denoted by $X_\ast$. The $i$-th row of $X_\ast$ with $1\le i\le n$  is the optimal weights for the $i$-th instance. It is worthy noting that the regularization encourages the similar instances to use the similar or even identical weights. If some rows of $X_\ast$  are identical, it means that the corresponding instances belong to a cluster. As illustrated in Figure \ref{figure_example_house}, we can obtain different clustering results by varying $\alpha$\footnote{The details are presented in the empirical studies.}. When $\alpha$ is very small, each vertex represents a cluster. With the increase of $\alpha$, more vertices are fused into a cluster.

We explain the model by using an example which is illustrated in Figure \ref{figure_illustration_triangle_lasso}. The vertex $v_5$ is profiled by using noisy data. As we have shown, we obtain
\begin{eqnarray}
\setlength{\abovedisplayskip}{0pt}
\setlength{\belowdisplayskip}{0pt}
\nonumber
g_{1,2} &=& \lVert  X_1 - X_2  \rVert \\ \nonumber
g_{1,4} &=& \lVert  X_1 - X_4  \rVert  + \lVert  X_1 - X_5  \rVert + \lVert  X_4 - X_5  \rVert \\ \nonumber
g_{1,5} &=& \lVert  X_1 - X_5 \rVert + \lVert  X_1 - X_4  \rVert + \lVert  X_5 - X_4  \rVert\\ \nonumber
g_{3,4} &=& \lVert  X_3 - X_4  \rVert \\ \nonumber
g_{4,5} &=& \lVert  X_4 - X_5  \rVert  + \lVert  X_4 - X_1  \rVert + \lVert  X_5 - X_1  \rVert.
\end{eqnarray} The regularized term is:
\begin{eqnarray}
\setlength{\abovedisplayskip}{0pt}
\setlength{\belowdisplayskip}{0pt}
\label{equa_triangle_lasso_large_regularization}
&&\sum\limits_{e_{ij}\in \mathcal{E}}g_{i,j}(X_i,X_j)\\ \nonumber
 &=& \lVert  X_1 - X_2  \rVert + 3 \lVert  X_1 - X_4  \rVert \\ \nonumber 
&+& 3\lVert  X_1 - X_5  \rVert + \lVert  X_3 - X_4  \rVert +  3 \lVert  X_4 - X_5  \rVert.
\end{eqnarray} For similarity, we consider the case of no weights for edges, and further let $\alpha = 1$. $Q$ is 
\begin{eqnarray}
\setlength{\abovedisplayskip}{0pt}
\setlength{\belowdisplayskip}{0pt}
\nonumber
Q=\begin{pmatrix}
 1&  -1&  0& 0 & 0\\ 
 3&  0&  0&  -3 & 0\\ 
 3&  0& 0 & 0 & -3\\ 
 0&  0&  1&  -1 & 0\\ 
 0&  0&  0&  3& -3
\end{pmatrix}.
\end{eqnarray}

\begin{figure}[t]
\centering 
\includegraphics[width=0.48\columnwidth]{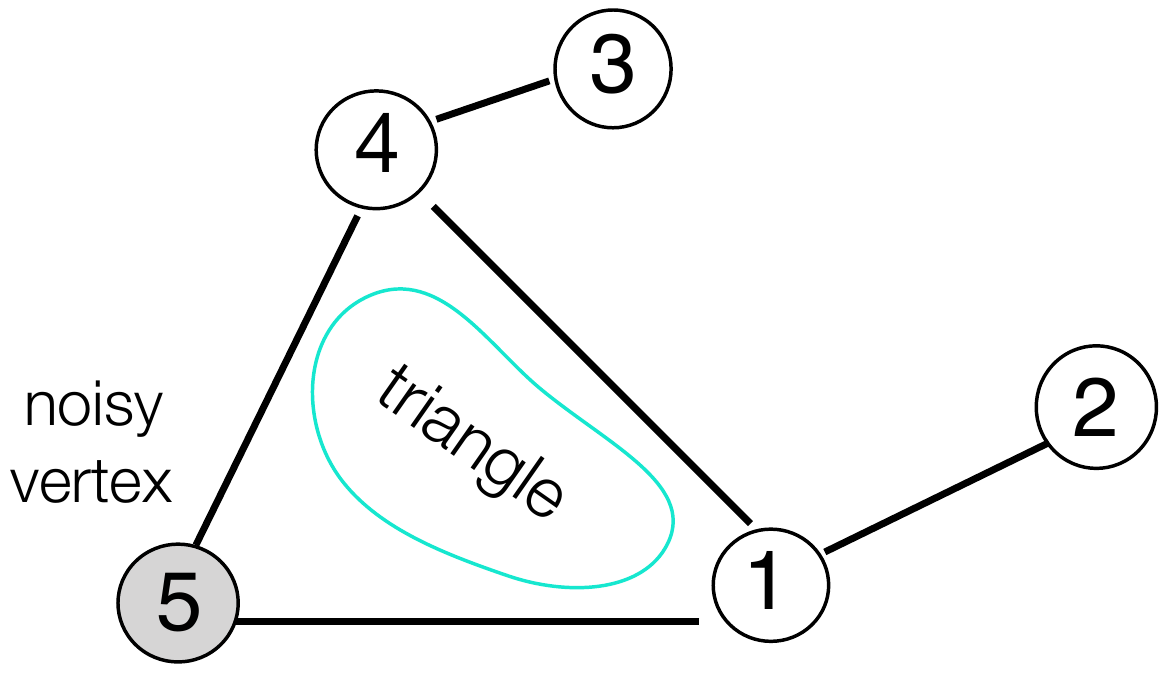}
\caption{If an edge exists in many triangles, its vertices are penalized more than other vertices. Thus, their weights tends to be more similar or even identical than others.}
\label{figure_illustration_triangle_lasso}
\end{figure}

As illustrated in Figure \ref{figure_illustration_triangle_lasso}, the vertices $v_1$, $v_4$ and $v_5$ exist in a triangle. The difference of their weights, namely $\lVert X_1-X_4\rVert$, $\lVert X_1-X_5\rVert$ and $\lVert X_4-X_5\rVert$   are penalized more than others.  The large penalization on the difference of $\lVert X_1-X_5\rVert$ and $\lVert X_4-X_5\rVert$ makes $X_5$ is close to $X_1$ and $X_4$.  Although $v_5$ is profiled by noisy data, we can still find its similar counterparts $v_1$, $v_4$. More generally, if some instances have missing values, those values are usually filled by using the mean value, the maximal value, the minimal value of the corresponding features, or the constant $0$. Comparing with the true values, those estimated values lead to noise. The noise impairs the performance of many classic methods when conducting data analysis tasks on those values directly. Note that triangle lasso does not only use the values, but also use the relation between different instances. If the vertices have many common neighbors in the graph, they tend to be similar even though they are represented by using noisy values. That is the reason why triangle lasso is robust to the imperfect data. 

Triangle lasso is a general and robust framework to simultaneously conduct clustering and optimization for various tasks. The whole workflow is presented in Figure \ref{figure_illustration_triangle_lasso_framework}. First, a graph is constructed to represent the dataset. Second, we obtain a convex optimization problem by formulating a specific data analysis task to the triangle lasso. Third, we provide two methods to solve the triangle lasso. Finally, we obtain the solution of the triangle lasso, and use it to complete data analysis tasks. 
Note that the graph or network datasets are the main targeting datasets for triangle lasso. For a graph or network dataset, $Q$ can be obtained trivially. Otherwise, we  represent the dataset as a graph as follows:\\
\textbf{Case 1:} If the dataset does not contain the imperfect data (missing, noisy, or unreliable values), we run $K$-Nearest Neighbours (KNN) method to find the $K$ nearest neighbours for each an instance. After that, we can obtain the graph by the following rules.
\begin{itemize}
\item Each instance is denoted by a vertex.
\item If an instance is one of the $K$ nearest neighbours of the other instance, then the vertices corresponding to them are connected by an edge. 
\end{itemize}
\textbf{Case 2:} If the dataset contains imperfect data, or contains redundant features in the high dimensional scenarios, we run the dimension reduction methods such as Principal Component Analysis (PCA) or feature selection to improve the quality of the dataset. Then, as mentioned above, we use KNN to find the $K$ nearest neighbours for each an instance, and obtain the graph.  
The procedure is suitable to both supervised learning and unsupervised learning.  Additonally, the $Q$ matrix in  (\ref{equa_triangle_lasso}) plays an essential role in the triangle lasso. Each element of $Q$ contains $\alpha$, $w_{ij}$ and $q_{ij}$. $\alpha$ is a hyper-parameter which needs to be given before optimizing the formulation. $w_{ij}$ and $q_{ij}$ are closely related to the graph.     When a dataset is represented as a graph, we need to determine $w_{ij}$ and $q_{ij}$ in order to obtain $Q$. $w_{ij}$ is the weight of the edge $e_{ij}$, which can be used to measure the importance of the edge. Some literatures recommend $w_{ij} = \exp(-\varrho \left \lVert X_i - X_j \right\rVert_2^2)$ where $\varrho $ is a non-negative constant \cite{Chi:2013ey,Chen:2015kn,Chi:2016bu}. When $\varrho = 0$, it represents the uniform weights. When   $\varrho > 0$, it represents the Gaussian kernel.  Besides, $q_{ij}$   measures the similiarity of nodes $v_i$ and $v_j$ due to their common adjacent nodes. 

\begin{figure}[t]
\centering 
\includegraphics[width=0.95\columnwidth]{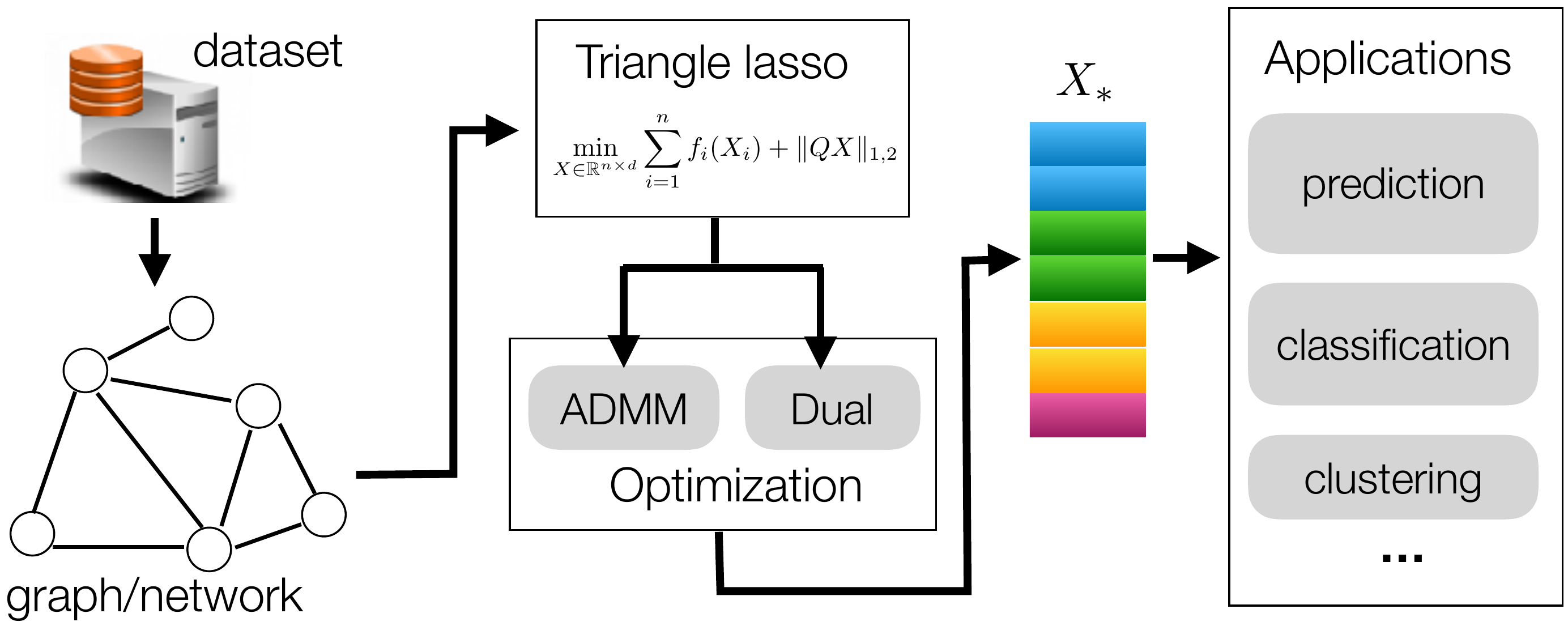}
\caption{The illustration of the workflow of triangle lasso. }
\label{figure_illustration_triangle_lasso_framework}
\end{figure}

\subsection{Applications}
The previous researches including network lasso and convex clustering are the special cases of the triangle lasso. If we force $g_{i,j}(X_i,X_j) = g_0(X_i,X_j)$, the triangle lasso degenerates to the network lasso. If we further force $f_i(X_i) = \lVert  X_i-A_i \rVert_2^2$, the triangle lasso degenerates to the convex clustering.  Specifically, we take the ridge regression and convex clustering as examples to illustrated triangle lasso in more details. 

\textbf{Ridge regression.} In a classic ridge regression task, the loss function is 
\begin{eqnarray}
\setlength{\abovedisplayskip}{0pt}
\setlength{\belowdisplayskip}{0pt}
\nonumber
\min_{\tiny x\in \mathbb{R}^{1\times d}} \frac{1}{n}\sum\limits_{i=1}^n \lVert  A_i x^\mathrm{T} - y_i \rVert_2^2+\gamma \lVert  x  \rVert^2_2.
\end{eqnarray}  Here, $A_i\in\mathbb{R}^{1\times d}$ represents the $i$-th instance in the data matrix, $y\in\mathbb{R}^{n\times1}$ is the response matrix, and $x$ is the need-to-learn weight.  $n$ represents the number of instances, and $\gamma$ with $\gamma > 0$ is the regularization coefficient to avoid overfitting. Note that $\gamma$ is a hyper-parameter introduced by the formulation of the ridge regression, not introduced by triangle lasso.
 We thus instantiate (\ref{equa_triangle_lasso}) as:
\begin{eqnarray}
\setlength{\abovedisplayskip}{0pt}
\setlength{\belowdisplayskip}{0pt}
\nonumber
\min_{\tiny X \in \mathbb{R}^{n\times d}} \frac{1}{n}\sum\limits_{i=1}^n\lVert  A_i X_i^\mathrm{T}  - y_i \rVert_2^2+\gamma \lVert  X  \rVert^2_F+\lVert  Q X  \rVert_{1,2}.
\end{eqnarray} Here, $A\in\mathbb{R}^{n\times d}$ is the stack of  the instances $A_i$ with $1\le i \le n$. $X$ is the stack of weights of those instances. That is, the $i$-th row of $X$, namely $X_i\in\mathbb{R}^{1\times d}$ is the weight of  $A_i$. In this case, 
\begin{eqnarray}
\setlength{\abovedisplayskip}{0pt}
\setlength{\belowdisplayskip}{0pt}
\nonumber
f_i(X_i,y_i) = \frac{1}{n} \lVert  A_i X_i^\mathrm{T}  - y_i \rVert_2^2+\frac{\gamma}{n} \lVert  X_i  \rVert^2_2.
\end{eqnarray}

\textbf{Convex clustering.} In a convex clustering task, the loss function is:
\begin{eqnarray}
\setlength{\abovedisplayskip}{0pt}
\setlength{\belowdisplayskip}{0pt}
\nonumber
\min_{\tiny X \in \mathbb{R}^{n\times d}}  \lVert  X - A \rVert_F^2+\alpha \sum\limits_{1\le i < j \le m}\lVert  X_i - X_j \rVert_2.
\end{eqnarray} Here,  $X$ is the need-to-learn weights. $m$ is the number of edges in the graph. $\alpha$ with $\alpha>0$  is used to control the number of clusters. Note that $\alpha$ is a parameter in the formulation of convex clustering, which can be varied to control the number of clusters. Different from the case of ridge regression,  $\alpha$ is usually increased heuristically in order to obtain a cluster path. Additionally, the $i$-th row of the optimal $X_\ast$ is the label of the instance $A_i$. If two instances $A_i$ and $A_j$ have the identical labels, it means that they belong to a cluster.  In this case,
\begin{eqnarray}
\setlength{\abovedisplayskip}{0pt}
\setlength{\belowdisplayskip}{0pt}
\nonumber
f_i(X_i) = \lVert  X_i-A_i \rVert_2^2.
\end{eqnarray} The final formulation of convex clustering is:
\begin{eqnarray}
\setlength{\abovedisplayskip}{0pt}
\setlength{\belowdisplayskip}{0pt}
\nonumber
\min_{\tiny X \in \mathbb{R}^{n\times d}}  \lVert  X - A \rVert_F^2+\lVert  QX \rVert_{1,2}.
\end{eqnarray}

Note that triangle lasso generally outperforms the classic convex clustering on recovering the correct clustering membership. We present more explanations from two views.
\begin{itemize}
\item Intuitively, triangle lasso uses the sum-of-norms regularization to obtain the clustering result, which is similar to the convex clustering. On the other hand, triangle lasso considers the neighbouring information of vertices, and uses it in the regularization. Since network science has claimed that the neighbours of vertices is essential to measure its importance in a graph \cite{Prell:2011:SNA:2222515}, triangle lasso has advantages on finding the similarity among instances over the classic convex clustering.  
\item Mathematically, triangle lasso gives large weights to a regularized item (see the equation (\ref{equa_triangle_lasso_large_regularization})), if they have many common neighbours. That is, such the regularized item is punished more than other items during the optimization procedure, which makes the vertices tend to be similar or even identical. This is different from the convex clustering because that convex clustering views each a regularized item equally, which ignores their neighbouring relationship.  
\end{itemize}

\textbf{Demo example.} To make it more clear, we take the house price prediction as a demo example to explain triangle lasso. This example is one of empirical studies in Section \ref{sect_empirical_studies}. We need to predict the price of houses in the Greater Sacramento area by using a ridge regression model. Our target is to learn the weight for each house.  Generally, the houses, which are located to a district, should use similar or identical weights. Those located in different districts should use different weights.  As illustrated in Figure \ref{figure_example_house}, triangle lasso will yield a weight for each house, and those weights can be used to as a label to obtain multiple clusters. The houses belonging to a cluster use an identical weight. We can adjust $\alpha$ to obtain different number of clusters. With the increase of $\alpha$, more houses are fused to a cluster.

\subsection{Symbols and their notations}
To make it easy to read, we present the symbols and their notations in  Table \ref{table_symbols}. Since the vector operation is usually easier to be understood and performed than the matrix operation. We tend to use vector operation replacing of the matrix operation in the paper equivalently. In other words, when we need to handle a matrix, we usually use its column stacking vectorization replacing of itself. For example, when we need to obtain the gradient of with respect to a matrix, we usually use $\partial f(vec(X))$ to replace $\partial f(X)$ for simplicity. In the paper, a matrix is viewed equivalent to its vectorization. For example, $f(X)$ is equivalent to $f(vec(X))$ because that we can transform them without any ambiguity. Finally, we use  the notation $f(X)$ in both supervised and unsupervised learning tasks for math brevity.

\begin{table}[!h]
\centering
\caption{The symbols and their notations}
\label{table_symbols}
\begin{tabular}{c|c}
\hline 
Symbols & Notations\tabularnewline
\hline 
\hline 
$\mathcal{V}$ & The vertex set containing $n$ vertices\tabularnewline
\hline 
$v_{i}$ & The $i$th vertex\tabularnewline
\hline 
$\mathcal{E}$ & The edge set containing $m$ edges\tabularnewline
\hline 
$e_{ij}$ & The edge connecting $v_i$ and $v_j$\tabularnewline
\hline 
$A$ & The data matrix\tabularnewline
\hline 
$A_i$ & The $i$-th instance\tabularnewline
\hline 
$y_i$ & The response of $A_i$\tabularnewline
\hline 
$X$ & The weight matrix\tabularnewline
\hline 
$X_i$ & The weights of the $v_i$\tabularnewline
\hline 
$N(X_i)$ & The neighbours of $X_i$\tabularnewline
\hline 
$w_{ij}$ & The weight corresponding to the edge $e_{ij}$\tabularnewline
\hline 
\multirow{1}{*}{$f^\ast(\cdot)$} & The convex conjugate of $f$\tabularnewline
\hline 
$vec(\cdot)$ & The column stacking vectorization\tabularnewline
\hline 
$\otimes$ & The kronecker product\tabularnewline
\hline 
$\circ$ & The element-wise product\tabularnewline
\hline 
$\parallel \cdot \parallel$ & The $l_2$ norm of a matrix defaultly\tabularnewline
\hline 
$\parallel \cdot \parallel_F$ & The Frobenius norm of a matrix\tabularnewline
\hline 
$\parallel \cdot \parallel_\ast$ & The dual norm\tabularnewline
\hline 
$\mathbf{1}$ & The  matrix whose elements are $1$\tabularnewline
\hline 
$I_d$ & The $d \times d$ unit matrix\tabularnewline
\hline 
$\alpha$ & The regularization coefficient\tabularnewline
\hline 
$t$ & The $t$th iteration of ADMM\tabularnewline
\hline 
$\rho$ & The step length of ADMM\tabularnewline
\hline 
$\partial$ & The sub-gradient operator\tabularnewline
\hline 
$\lambda, U$ & The dual variable\tabularnewline
\hline 
$\lambda_i$ & The $i$th row of $\lambda$\tabularnewline
\hline 
\textbf{Prox($\cdot$)} & The proximal operator\tabularnewline
\hline 
\end{tabular}
\end{table}

\section{ADMM method for the moderately accurate solution}
\label{sect_admm}
In this section, we present our ADMM method to solve triangle lasso. First, we present the details of our ADMM method as a general framework. Second, we present an example to make our method easy to understand. Finally, we discuss the convergence and the stopping criterion of our method.

\subsection{Details}
Before presentation of our method, we need to re-formulate the unconstrained optimization (\ref{equa_triangle_lasso}) to be a constrained  problem equivalently:
\begin{eqnarray}
\setlength{\abovedisplayskip}{0pt}
\setlength{\belowdisplayskip}{0pt}
\label{equa_triangle_lasso_constrained}
\min_{\tiny X\in \mathbb{R}^{n\times d}, Z\in \mathbb{R}^{m\times d}} f(X)+g(Z)
\end{eqnarray} subject to: $$QX-Z=0$$ where $f(X) = \sum\limits_{i=1}^n f_{i}(X_i)$ and $g(Z) = \lVert Z\rVert_{1,2}$. Suppose the Lagrangian dual variable is denoted by $U$ with $U\in \mathbb{R}^{n\times d}$. Its augmented Lagrangian multiplier is:
\begin{eqnarray}
\setlength{\abovedisplayskip}{0pt}
\setlength{\belowdisplayskip}{0pt}
\nonumber
&&L_\rho(X,Z, U) \\ \nonumber
&=& f(X) \mathrm{+} g(Z) \mathrm{+} \mathbf{1}_{1\times n}(U\circ(QX\mathrm{-}Z))\mathbf{1}_{p\times 1}\mathrm{+}\frac{\rho}{2}\lVert QX\mathrm{-}Z  \rVert_F^2
\end{eqnarray} where $\rho$ is a positive number.

\textbf{Update of $X$}. The basic update of $X$ is:
\begin{eqnarray}
\setlength{\abovedisplayskip}{0pt}
\setlength{\belowdisplayskip}{0pt}
\nonumber
X^{t+1} = \mathop{\arg\min}\limits_{X} L_\rho(X,Z^t, U^t)
\end{eqnarray} where $t$ represents the $t$-th iteration. Suppose $h(X) = \mathbf{1}_{1\times n}(U^t\circ(QX\mathrm{-}Z^t))\mathbf{1}_{p\times 1}+\frac{\rho}{2}\lVert QX\mathrm{-}Z^t  \rVert_F^2$. Discarding constant items, we obtain
\begin{eqnarray}
\setlength{\abovedisplayskip}{0pt}
\setlength{\belowdisplayskip}{0pt}
\label{equa_update_X}
X^{t+1} = \mathop{\arg\min}\limits_{X} f(X) + h(X).
\end{eqnarray} Apparently, $h(X)$  is strongly convex and smooth. Therefore, the hardness   of  the update of $X$ is dominated by $f(X)$.

\underline{Convex case.} If $f(X)$ is convex, it is easy to know that $f(X)+h(X)$ is convex too. Thus, it is not difficult to obtain $X^{t+1}$ by solving the convex optimization (\ref{equa_update_X}). For a general convex case, we can update $X$ by solving the following equality:
\begin{eqnarray}
\setlength{\abovedisplayskip}{0pt}
\setlength{\belowdisplayskip}{0pt}
\nonumber
\partial f(X^{t+1}) + \partial h(X^{t+1}) = 0
\end{eqnarray}  where $\partial$ represents the sub-gradient operator.

\underline{Non-convex case.} If $f(X)$ is non-convex, $f(X)+h(X)$ may not be convex. Thus, the global minimum of (\ref{equa_update_X}) is not guaranteed, and we have to obtain a local minimum. Considering that the non-convex optimization may be much more difficult than the convex case, the update of $X$ may be time-consuming. Since the Lagrangian dual of (\ref{equa_update_X}) is always convex, we update $X$ via the dual problem of (\ref{equa_update_X}). 

Before presentation of the method, we need to transform (\ref{equa_update_X}) to be a constrained problem equivalently:
\begin{eqnarray}
\setlength{\abovedisplayskip}{0pt}
\setlength{\belowdisplayskip}{0pt}
\nonumber
\min_{X,Y} f(Y) + h(X)
\end{eqnarray} subject to: $$Y-X=0.
$$ Its Lagrangian multiplier is:
\begin{eqnarray}
\setlength{\abovedisplayskip}{0pt}
\setlength{\belowdisplayskip}{0pt}
\nonumber
L(Y,X,\lambda) &=& f(Y) + h(X) + \mathbf{1}_{1\times n}(\lambda \circ (Y-X))\mathbf{1}_{p\times 1} \\ \nonumber
&=& f(vec(Y)) + h(vec(X)) + vec^{\mathrm{T}}(\lambda)vec(Y) \\ \nonumber
&-& vec^{\mathrm{T}}(\lambda)vec(X)
\end{eqnarray} where $vec(\cdot)$ represents the column stacking vectorization of a matrix. Therefore, the Lagrangian dual is:
\begin{eqnarray}
\setlength{\abovedisplayskip}{0pt}
\setlength{\belowdisplayskip}{0pt}
\nonumber
D(vec(\lambda)) &=& \inf_{vec(Y)} f(vec(Y)) + vec^{\mathrm{T}}(\lambda)vec(Y) \\ \nonumber 
&+& \inf_{vec(X)} h(vec(X)) -vec^{\mathrm{T}}(\lambda)vec(X)\\ \nonumber
&=& -f^\ast(-vec(\lambda)) - h^\ast(vec(\lambda))
\end{eqnarray} where $f^\ast(\cdot)$ is the convex conjugate of $f(\cdot)$, and $h^\ast(\cdot)$ is the convex conjugate of $h(\cdot)$. Generally, the convex conjugate function $f^\ast(x)$ is defined as $f^\ast(y) = \sup_{x}(y^{\mathrm{T}}x - f(x))$. Thus, the dual problem is:
\begin{eqnarray}
\setlength{\abovedisplayskip}{0pt}
\setlength{\belowdisplayskip}{0pt}
\nonumber
\min_{\lambda} f^\ast(-vec(\lambda)) + h^\ast(vec(\lambda)).
\end{eqnarray} Since the dual problem is always convex, it is easy to obtain its global minimum $\lambda_\ast$. According to the KKT conditions, we obtain $X^{t+1}$ by solving:
\begin{eqnarray}
\setlength{\abovedisplayskip}{0pt}
\setlength{\belowdisplayskip}{0pt}
\nonumber
\partial h(X^{t+1}) - \lambda_\ast=0.
\end{eqnarray} 

In some non-convex cases of $f(X)$, we can still obtain the global minimum when there is no duality gap, i.e. strong duality. There are various methods to verify whether there is duality gap. It is out of the scope of the paper, we recommend readers to refer the related books \cite{Bertsekas20046}.

\textbf{Update of $Z$}. $g(Z)$ is a sum-of-norms regularization, which is convex but not smooth. It is not differentiable when arbitrary two rows of $Z$ are identical. Unfortunately, we encourage the rows of $Z$ becomes identical in order to find the similar instances. Therefore, it is non-trivial  to obtain the global minimum $Z_\ast$ in the triangle lasso. In the paper, we obtain the closed form of $Z_\ast$ via the proximal operator of a sum-of-norms function.  

The basic update of $Z$ is 
\begin{eqnarray}
\setlength{\abovedisplayskip}{0pt}
\setlength{\belowdisplayskip}{0pt}
\nonumber
&&Z^{t+1} =\mathop{\arg\min}\limits_{Z} L_\rho(X^{t+1},Z, U^t)\\ \nonumber
&=& \mathop{\arg\min}\limits_{Z} g(Z) - \mathbf{1}_{1\times n}(U^t\circ Z)\mathbf{1}_{p\times 1}+\frac{\rho}{2}\lVert QX^{t+1}\mathrm{-}Z  \rVert_F^2.
\end{eqnarray} Discarding the constant item, we obtain
\begin{eqnarray}
\setlength{\abovedisplayskip}{0pt}
\setlength{\belowdisplayskip}{0pt}
\nonumber
&&vec(Z^{t+1}) \\ \nonumber
&=&\mathop{\arg\min}\limits_{vec(Z)}  g(vec(Z)) - vec^{\mathrm{T}}(U^t)vec( Z)\\ \nonumber
&+&\frac{\rho}{2} \left(vec^{\mathrm{T}}(Z)vec(Z) - 2vec^{\mathrm{T}}(QX^{t+1})vec(Z) \right) \\ \nonumber
&=&\mathop{\arg\min}\limits_{vec(Z)}  g(vec(Z)) \\ \nonumber
&+& \frac{\rho}{2}\left \lVert vec(Z) - \left(vec\left(QX^{t+1}\right)+\frac{1}{\rho}vec(U^t) \right)\right \rVert^2\\ \nonumber
&=&\mathbf{Prox}_{\rho,g}\left(  vec(QX^{t+1}) + \frac{1}{\rho} vec(U^t)  \right).
\end{eqnarray}  $\mathbf{Prox}_{\rho,g}(\cdot)$ is the proximal operator of $g(\cdot)$ with the efficient $\rho$ which is defined as: $\mathbf{Prox}_{\nu,\phi}(v) = \mathop{\arg\min}\limits_{x} \phi(x) + \frac{\nu}{2}\lVert x-v \rVert^2$. Considering $g(Z)$ is a sum-of-norms function, its proximal operator has a closed form \cite{Parikh2014Proximal}, that is:
\begin{eqnarray}
\setlength{\abovedisplayskip}{0pt}
\setlength{\belowdisplayskip}{0pt}
\nonumber
&&Z_i^{t+1} = [\mathbf{Prox}_{\rho,g}(Z^{t+1})]_i \\ \nonumber 
&=& \left( 1-\frac{1}{\lVert \left( \rho QX^{t+1} +U^t \right)_i \rVert}  \right)_{+} (QX^{t+1}+\frac{1}{\rho} U^t)_i \\ \nonumber
&=& \max\left\{0, 1-\frac{1}{\lVert \left( \rho QX^{t+1} +U^t \right)_i \rVert}  \right\} \left(QX^{t+1}+\frac{1}{\rho} U^t \right)_i.
\end{eqnarray}   Here, the subscript `$+$' represents non-negative value for each element in the matrix. The subscript `$i$' with $1\le i \le m$ represents the $i$-th row of a matrix. If some elements are negative, their values will be set to be zeros. Otherwise, the positive value will be reserved.  

\textbf{ Update of $U$.} $U$ is updated by the following rule:
\begin{eqnarray}
\setlength{\abovedisplayskip}{0pt}
\setlength{\belowdisplayskip}{0pt}
\label{equa_update_U}
U^{t+1} = U^t + \rho (QX^{t+1}-Z^{t+1}).
\end{eqnarray}

\subsection{Examples}
To make our ADMM easy to understand, we take the ridge regression as an example to show the details. As we have shown in Section \ref{sect_formulation}, the optimization objective function is:
\begin{eqnarray}
\setlength{\abovedisplayskip}{0pt}
\setlength{\belowdisplayskip}{0pt}
\nonumber
\min_{\tiny X \in \mathbb{R}^{n\times d}} \sum\limits_{i=1}^n\lVert  A_i X_i^\mathrm{T}  - y_i \rVert_2^2+ \gamma \lVert X  \rVert_F^2 +  \lVert  Q X  \rVert_{1,2}.
\end{eqnarray} We thus obtain $f(X) = \sum\limits_{i=1}^n\lVert  A_i X_i^\mathrm{T}  - y_i \rVert_2^2+ \gamma \lVert X  \rVert_F^2$ which is convex and smooth. Therefore, the update of $X^{t+1}$ is to solve the following equalities:
\begin{eqnarray}
\setlength{\abovedisplayskip}{0pt}
\setlength{\belowdisplayskip}{0pt}
\nonumber
(A_i X_i^\mathrm{T}  - y_i )A_i + \gamma X_i = 0, {~~~~} 1\le i \le n.
\end{eqnarray} The update of $Z^{t+1}$ is independent to $f(X)$, and $U^{t+1}$ is easy to understand. We do not re-write them again.

\subsection{Convergence and stopping criterion}
\label{subsection_admm_stop_criterion}
When $f(X)$ and $g(Z)$ are convex, the ADMM method is convergent \cite{Boyd:2011}. Recently, many researches have investigated the convergence of ADMM \cite{Nishihara,tongda}.  But, it is non-trivial to obtain the convergence rate for a general $f(X)$ and $g(Z)$. In triangle lasso, $g(Z)$ is convex but not smooth. The convergence rate is impacted by the convexity of $f(X)$ and the matrix $Q$. Many previous researches have claimed that if $f(X)$ is smooth and $Q$ is  row full rank, the ADMM will obtain a linear convergence rate \cite{Hong2017}.  

 The basic ADMM has its stopping criterion \cite{Boyd:2011}. But, we can re-define the stopping criterion of ADMM in triangle lasso for some specific tasks to gain a high efficiency. Taking convex clustering as an example, we do not care the specific value of $X$. All we want to obtain is the clustering result. If two rows of $X$ are identical, the corresponding instances belong to a cluster.  If the clustering result keeps same between two iterations, we can stop the method when $X$ is close to the minimum. Finally, our ADMM method is illustrated in Algorithm \ref{algo_admm}.
\begin{algorithm}[t]
    \caption{ADMM for the triangle lasso}
    \label{algo_admm}
    \begin{algorithmic}[1]
        \Require The data matrix $A\in\mathbb{R}^{n\times d}$, and a positive $\alpha$. $t=0$.
        \State Initialize $X^0$, $Z^0$, and $U^0$.
        \For {Stopping criterion is not satisfied}
            \If {$f(X)$ is convex}
                \State Update $X^{t+1 }$ by solving $\partial f(X^{t+1}) + \partial h(X^{t+1}) = 0$.
            \EndIf
            \If {$f(X)$ is non-convex}
                \State $\lambda_\ast = \mathop{\arg\min}\limits_{\lambda} f^\ast(-vec(\lambda)) + h^\ast(vec(\lambda))$.
                \State Update $X^{t+1}$ by solving $\partial h(X^{t+1}) - \lambda^{\ast}=0$.
            \EndIf
            \State $Z_i^{t+1} \mathrm{=} \max\left\{0, 1\mathrm{-}\frac{1}{\lVert \left( \rho QX^{t+1} +U^t \right)_i \rVert}  \right\} \left(QX^{t+1}\mathrm{+}\frac{1}{\rho} U^t \right)_i$ with $1\le i \le m$.
            \State $U^{t+1} = U^t + \rho (QX^{t+1}-Z^{t+1})$.
            \State $t=t+1$;
        \EndFor
        \State\Return The final value of $X$.
    \end{algorithmic}
\end{algorithm}

\section{Dual method for the accurate solution}
\label{sect_dual}
Although our ADMM is efficient to yield a moderately  accurate solution, it is necessary to provide an efficient method to obtain the accurate solution in some applications. In the section, we transform (\ref{equa_triangle_lasso}) to be a second-order cone programming problem, and develop a method to solve it in the dual space. First, we first present the details of our Dual method. Second, we use an example to explain our dual method.

\subsection{Details}
We first re-formulate (\ref{equa_triangle_lasso}) to be a constrained optimization problem equivalently.  
\begin{eqnarray}
\setlength{\abovedisplayskip}{0pt}
\setlength{\belowdisplayskip}{0pt}
\nonumber
\min_{X,Z} f(X) + g(Z)
\end{eqnarray} subject to: $$vec(QX)-vec(Z)=0.$$ Its Lagrangian multiplier is:
\begin{eqnarray}
\setlength{\abovedisplayskip}{0pt}
\setlength{\belowdisplayskip}{0pt}
\nonumber
L(X,Z,\lambda) = f(X)+g(Z)+vec^{\mathrm{T}}(\lambda)(vec(QX)-vec(Z)). 
\end{eqnarray} Thus, the dual optimization objective function is:
\begin{eqnarray}
\setlength{\abovedisplayskip}{0pt}
\setlength{\belowdisplayskip}{0pt}
\nonumber
D(\lambda) &=& \inf_{X}f(vec(X))+vec^{\mathrm{T}}(\lambda) vec(QX)  \\ \nonumber  
&+& \inf_{Z} g(Z) - vec^{\mathrm{T}}(\lambda) vec(Z) \\ \nonumber
&=& \inf_{X} f(vec(X)) + vec^{\mathrm{T}}(\lambda) ((I_{d} \otimes Q)vec(X)) \\ \nonumber
&+& \inf_{Z} \sum\limits_{i=1}^m g_i(Z_i) - \lambda_i Z_i^\mathrm{T}\\ \nonumber
&=& -f^{\ast}(-vec^{\mathrm{T}}(\lambda)(I_d \otimes Q) ) - \sum\limits_{i=1}^m g_i^\ast (\lambda_i)
\end{eqnarray} Here, $Z_i$ and $\lambda_i$ represent the $i$-th row of $Z$ and $\lambda$, respectively. $g_i(\lambda_i) = \lVert \lambda_i \rVert$ holds, and $g_i^\ast(\lambda_i)$ is its convex conjugate.  Thus, we obtain
\begin{eqnarray}
\setlength{\abovedisplayskip}{0pt}
\setlength{\belowdisplayskip}{0pt}
\nonumber
g_i^\ast (\lambda_i) = \left\{ 
\begin{array}{ll}
0, {~~~~~} \lVert  \lambda_i \rVert_\ast \le 1\\
\infty, {~~~~} otherwise
\end{array} \right.
\end{eqnarray}
where $\lVert \cdot \rVert_\ast$ denotes the dual norm of $\lVert \cdot \rVert$. Since the dual norm  of the $l_2$ norm is still the $l_2$ norm, its dual problem is:
\begin{eqnarray}
\setlength{\abovedisplayskip}{0pt}
\setlength{\belowdisplayskip}{0pt}
\label{equa_triangle_lasso_dense}
\min_\lambda   f^\ast(-vec^{\mathrm{T}}(\lambda)(I_d \otimes Q))
\end{eqnarray} subject to:
\begin{eqnarray}
\setlength{\abovedisplayskip}{0pt}
\setlength{\belowdisplayskip}{0pt}
\nonumber
\lVert  \lambda_i \rVert_2 \le 1, {~~~~~} 1\le i\le m.
\end{eqnarray} After that, we can obtain the optimal $X_\ast$ by solving 
\begin{eqnarray}
\setlength{\abovedisplayskip}{0pt}
\setlength{\belowdisplayskip}{0pt}
\label{equa_recover_x_dual}
\partial f(vec(X_\ast)) + (I_d \otimes Q)^\mathrm{T} vec(\lambda_\ast) = 0.
\end{eqnarray} Here, $\lambda_\ast$ is the minimizer of (\ref{equa_triangle_lasso_dense}).  Since the  conjugate function $f^\ast(\cdot)$ is always convex no matter whether $f(\cdot)$ is convex. The dual problem (\ref{equa_triangle_lasso_dense}) is easier to be solved than the primal problem. If there is no duality gap between (\ref{equa_triangle_lasso}) and (\ref{equa_triangle_lasso_dense}), the global minimum of the primal problem  (\ref{equa_triangle_lasso})  can be obtained from the solution of the dual problem (\ref{equa_triangle_lasso_dense})  according to (\ref{equa_recover_x_dual}).

\begin{theorem}
\label{theorem_conjugate_sum}
The conjugate of the sum of the independent convex functions is the sum of their conjugates. Here, "independent" means that they have different variables \cite{Boyd:2004}.
\end{theorem}

According to Theorem \ref{theorem_conjugate_sum}, if $f(x)$ is separable, that is, $f(X_1, ..., X_i, ..., X_m) = \sum\limits_{i=1}^m f_i(X_i)$, we have $f^\ast(X_1, ..., X_i, ..., X_m) = \sum\limits_{i=1}^m f^\ast_i(X_i)$. We can obtain the solution of  (\ref{equa_triangle_lasso_dense}) by solving each component $f^\ast_i(\cdot)$ with $1\le i\le m$ independently. But, when the $f(x)$ is not separable, we have to solve (\ref{equa_triangle_lasso_dense}) as an entire problem. Unfortunately,  it may be time-consuming to solve (\ref{equa_triangle_lasso_dense}) for a large dense graph because that we have to optimize a large number of variables, i.e. $O(md)$. But, we can divide the graph to multiple sub-graphs, and solve (\ref{equa_triangle_lasso_dense})  for each sub-graph. Repeating those steps for different graph partitions, we can refine the final solution  of (\ref{equa_triangle_lasso_dense}). Finally, the details of our dual method is illustrated in Algorithm \ref{algo_dc_triangle_lasso}.
\begin{algorithm}[t]
    \caption{Dual method for the triangle lasso}
    \label{algo_dc_triangle_lasso}
    \begin{algorithmic}[1]
        \Require The data matrix $A\in\mathbb{R}^{n\times d}$, the graph $\mathcal{G}$, and a positive $\alpha$. 
        \State Solve (\ref{equa_triangle_lasso_dense}) for $\mathcal{G}$, and obtain $\lambda_\ast$.
        \State Obtain the optimal $X_\ast$ by solving (\ref{equa_recover_x_dual}).
        \end{algorithmic}
\end{algorithm}

\subsection{Example}
To make it easy to understand, we take the ridge regression as an example to show the details of the method. As we have shown in Section \ref{sect_formulation}, the optimization objective function is:
\begin{eqnarray}
\setlength{\abovedisplayskip}{0pt}
\setlength{\belowdisplayskip}{0pt}
\nonumber
\min_{\tiny X \in \mathbb{R}^{n\times d}} \sum\limits_{i=1}^n\lVert  A_i X_i^\mathrm{T}  - y_i \rVert_2^2+ \gamma \lVert X  \rVert_F^2 +   \lVert  Q X  \rVert_{1,2}.
\end{eqnarray} We thus obtain
\begin{eqnarray}
\setlength{\abovedisplayskip}{0pt}
\setlength{\belowdisplayskip}{0pt}
\nonumber
f(vec(X)) &=& \sum\limits_{i=1}^n\lVert  A_i X_i^\mathrm{T}  - y_i \rVert_2^2+ \gamma \lVert X  \rVert_F^2 \\ \nonumber
&=&  vec^{\mathrm{T}}(X)\Omega vec(X) - 2 \Phi vec(X).
\end{eqnarray} Here, $\Delta = (\mathbf{1}_{1\times d} \otimes I_n)\text{diag}(vec(A))$, $\Omega = \Delta^{\mathrm{T}}\Delta + \gamma I_{nd}$ and $\Phi = y^{\mathrm{T}}\Delta$. $\text{diag}(vec(A))$ yields a diagional matrix consisting of $vec(A)$. Discarding the constant item, we obtain 
\begin{eqnarray}
\setlength{\abovedisplayskip}{0pt}
\setlength{\belowdisplayskip}{0pt}
\nonumber
f^\ast(\theta) = \frac{1}{4} \left( \theta^{\mathrm{T}}  \Omega^{-1} \theta + 4 \Phi \Omega^{-1} \theta  \right).
\end{eqnarray} Substituting $\theta$ with $-(I_d \otimes Q^{T})vec(\lambda)$, we obtain the equivalent formulation is: 
\begin{eqnarray}
\setlength{\abovedisplayskip}{0pt}
\setlength{\belowdisplayskip}{0pt}
\nonumber
\min_{\lambda\in\mathbb{R}^{m\times d}} &&  vec^{\mathrm{T}}(\lambda)  (I_d \otimes Q)  \Omega^{-1}  (I_d \otimes Q)^{\mathrm{T}}   vec(\lambda) \\ \nonumber 
&-& 4 \Phi \Omega^{-1} (I_d \otimes Q)^{\mathrm{T}} vec(\lambda)
\end{eqnarray} subject to 
\begin{eqnarray}
\setlength{\abovedisplayskip}{0pt}
\setlength{\belowdisplayskip}{0pt}
\nonumber
\lVert  \lambda_i \rVert_2 \le 1, {~~~~~} 1\le i\le m.
\end{eqnarray}  After solving this equivalent optimization problem, we obtain the optimal $\lambda$, namely $\lambda_\ast$. Finally, the optimal $X$ is 
\begin{eqnarray}
\setlength{\abovedisplayskip}{0pt}
\setlength{\belowdisplayskip}{0pt}
\nonumber
vec(X_\ast) = \frac{1}{2} \Omega^{-1} \left(-(I_d \otimes Q^{\mathrm{T}})vec(\lambda_\ast)   + 2  \Phi^{\mathrm{T}}   \right).
\end{eqnarray}

\section{Complexity analysis}
\label{sect_time_complexity}

In this section, we analyze the time complexity of the proposed methods, i.e., the ADMM method and the dual method, for the case of convex $f(\cdot)$. 

\subsection{Time complexity of the ADMM method}

Consider the ADMM method. It is time-consuming for the calculation of the gradient rather than  the matrix multiplication. The time complexity due to the calculation of the gradient per iteration  is $O(nd)$. Note that the number of the iterations dominates the total time complexity of the ADMM method. For example, if the number of iterations is $T$, the total time complexity is $O(Tnd)$.  Generally, the large number of iterations leads to a relatively accurate solution, which leads to high time complexity.  Before presenting the time complexity formally, we introduce some new notations.  $w^t\in\RR^{(nd+2md)\times 1}$ yielded by ADMM at the $t$-th iteration is defined as

\begin{align}
\nonumber
w^t := \lrincir{vec^T(X^t);vec^T(Z^t);vec^T(\lambda^t)}.
\end{align} 

Given a vector $w\in\RR^{(nd+2md)\times 1}$, $\lrnorm{w}_H^2$ is defined as
\begin{align}
\nonumber
\lrnorm{w}_H^2 := w^THw
\end{align} where $H$ is defined by 
\begin{align}
\nonumber
H := \begin{pmatrix}
 \textbf{0}_{nd\times nd}&  & \\ 
 & \rho I_{md} & \\ 
 &  & \frac{1}{\rho} I_{md}
\end{pmatrix}.
\end{align} 
Thus, when $f(\cdot)$ is convex, the total time complexity of our ADMM method is presented as the following theorem.
\begin{theorem}
When our ADMM is convergent satisfying $\lrnorm{w^t - w^{t+1}}_H^2 \le \epsilon$, the total time complexity of our ADMM is $O\lrincir{\frac{nd}{\epsilon}}$.
\end{theorem}
\begin{proof}

\cite{He2015} proves that $\lrnorm{w^t - w^{t+1}}_H^2 \le \frac{1}{t+1} \lrnorm{w^0 - w_{\ast}}_H^2$ holds when the Douglas-Rachford ADMM is performed for $t$ iterations (Theorem $5.1$ in \cite{He2015}). Our ADMM is its special case when $f(\cdot)$ is convex. Thus, given an $\epsilon>0$, to obtain $\lrnorm{w^t - w^{t+1}}_H^2 \le \epsilon$, our ADMM needs to be run for $\frac{1}{\epsilon}\lrnorm{w^0 - w_{\ast}}_H^2-1$ iterations. Since the time complexity per iteration is $O(nd)$, and $\lrnorm{w^0 - w_{\ast}}_H^2$ is a constant, the total time complexity is $O\lrincir{\frac{nd}{\epsilon}}$.

\end{proof}

\subsection{Time complexity of the dual method}
Consider the dual method. Before presenting the details of the complexity analysis. Let us present some basic definitions, which are widely used to analyze the performance of an optimization method theoretically \cite{Nocedal1999Numerical,Shalev2014Understanding,ShalevShwartz:2012dz,Hazan2016Introduction}.

\begin{definition}[$\zeta$-smooth]
\nonumber
A function $f: \Xcal\mapsto \RR$ is $\zeta$ ($\zeta>0$) smooth, if and only if, for any vecoters $x\in\Xcal$ and $y\in\Xcal$, we have $f(y)\le f(x)+\lrangle{\nabla f(x), y-x} + \frac{\zeta}{2}\lrnorm{y - x}^2$.
\end{definition}
\begin{definition}[$\varsigma$-strongly convex]
\nonumber
A function $f: \Xcal\mapsto \RR$ is $\varsigma $ ($\varsigma >0$) strongly convex, if and only if, for any vectors $x\in\Xcal$ and $y\in\Xcal$, we have $f(y)\ge f(x)+\lrangle{\nabla f(x), y-x} + \frac{\varsigma}{2}\lrnorm{y - x}^2$.
\end{definition} 
\begin{definition}
\nonumber
If a function $f(\cdot)$ is $\zeta$-smooth and $\varsigma$-strongly convex, its condition number $\kappa$ is defined by $\kappa := \frac{\zeta}{\varsigma}$.
\end{definition}

There are many tasks whose optimization objective function is smooth and strongly convex. Those tasks include convex clustering, ridge regression, $l_2$ norm regularized logistic regression etc. We recommend to \cite{Shalev2014Understanding} for more details. When $f(\cdot)$ is $\zeta$-smooth and $\varsigma$-strongly convex, its convex conjugate function $f^{\ast}(\cdot)$ is thus $\frac{1}{\varsigma}$-smooth and $\frac{1}{\zeta}$-strongly convex (Lemma $2.19$ in \cite{ShalevShwartz:2012dz}). The condition number of $f^{\ast}(\cdot)$ is $\kappa = \frac{\zeta}{\varsigma}$. Additionally, there are various optimization methods to solve the dual problem (\ref{equa_triangle_lasso_dense}). Since Nesterov optimal method \cite{Nesterov:2004ic} is one of the widely used optimization methods, we use it to solve the dual problem.

\begin{theorem}
\nonumber
When the Nesterov optimal method is used to solve the dual problem (\ref{equa_triangle_lasso_dense}), and obtains $\lVert vec(\lambda^t) - vec(\lambda_{\ast}) \rVert \le \epsilon$ for a given positive $\epsilon$, then the total time complexity is $O\lrincir{\frac{md}{\epsilon}\sqrt{\kappa}}$. 
\end{theorem}
\begin{proof}

When we use Nesterov optimal method to solve the dual problem, the number of iterations is required to be $O(\frac{1}{\epsilon}\sqrt{\kappa})$ for $\lVert vec(\lambda^t) - vec(\lambda_{\ast}) \rVert \le \epsilon$ (Corollary $1$ in \cite{Lan:2009ea}). Furthermore, the Nesterov optimal method performs a gradient descent per iteration, which leads to $O(md)$ time complexity. Thus, the total time complexity is $O\lrincir{\frac{md}{\epsilon}\sqrt{\kappa}}$.

\end{proof}

\section{Empirical studies}
\label{sect_empirical_studies}
In this section, we conduct empirical studies to evaluate triangle lasso on the robustness and efficiency. First, we present the settings of the experiments. Second, we evaluate the robustness and efficiency of triangle lasso by conducting prediction tasks. Third, we evaluate the quality of the cluster path by conducting convex clustering with triangle lasso. After that, we evaluate the efficiency of our methods in various network topologies. Finally, we use triangle lasso to conduct community detection in order to show that triangle lasso is able to perform a general data analysis task.

\subsection{Settings}

\textbf{Model and algorithms.} As we have shown in the previous section, we conduct empirically studies by conducting ridge regression and convex clustering tasks. The weights of edges, i.e. $w_{ij}$ in $Q$ is set to be negatively proportional to the distance between the vertices.   All the algorithms are implemented by using Matlab 2015b and the solver CVX \cite{gb08}. The hardware is a server equipped with an i7-4790 CPU and $20$GB memory.   

The total compared algorithms are:
\begin{itemize}
\item \textbf{Network lasso \cite{Hallac:2015fy}.} This is the state-of-the-art method to conduct data analysis and clustering simultaneously. Both network lasso and triangle lasso can be used as a general framework. Thus, we compare the triangle lasso with it in the prediction tasks.
\item \textbf{AMA \cite{Chi:2013ey}.} This is the state-of-the-art method to conduct convex clustering. Convex clustering is a special case of the network lasso and triangle lasso. We compare triangle lasso with it in the convex clustering task.
\item \textbf{Triangle lasso-basic ADMM.} This is the basic version of ADMM which is used to solve triangle lasso in the evaluations. We use it as the baseline to compare our algorithms and other state-of-the-art methods.
\item \textbf{Triangle lasso-ADMM.} This is our proposed ADMM method to solve the triangle lasso. Since it is very fast, we use it to conduct each evaluations in default.
\item \textbf{Triangle lasso-Dual.} This is our proposed Dual method to solve the triangle lasso. It is not efficient when the dataset or the graph is large. We use it to conduct evaluations on some moderate graphs. As we have illustrated, triangle lasso is implemented by our ADMM method defaultly.
\end{itemize}

\begin{figure*}[!t]
\setlength{\abovecaptionskip}{0pt}
\setlength{\belowcaptionskip}{0pt}
\centering 
\subfigure[Best $\alpha$]{\includegraphics[width=0.49\columnwidth]{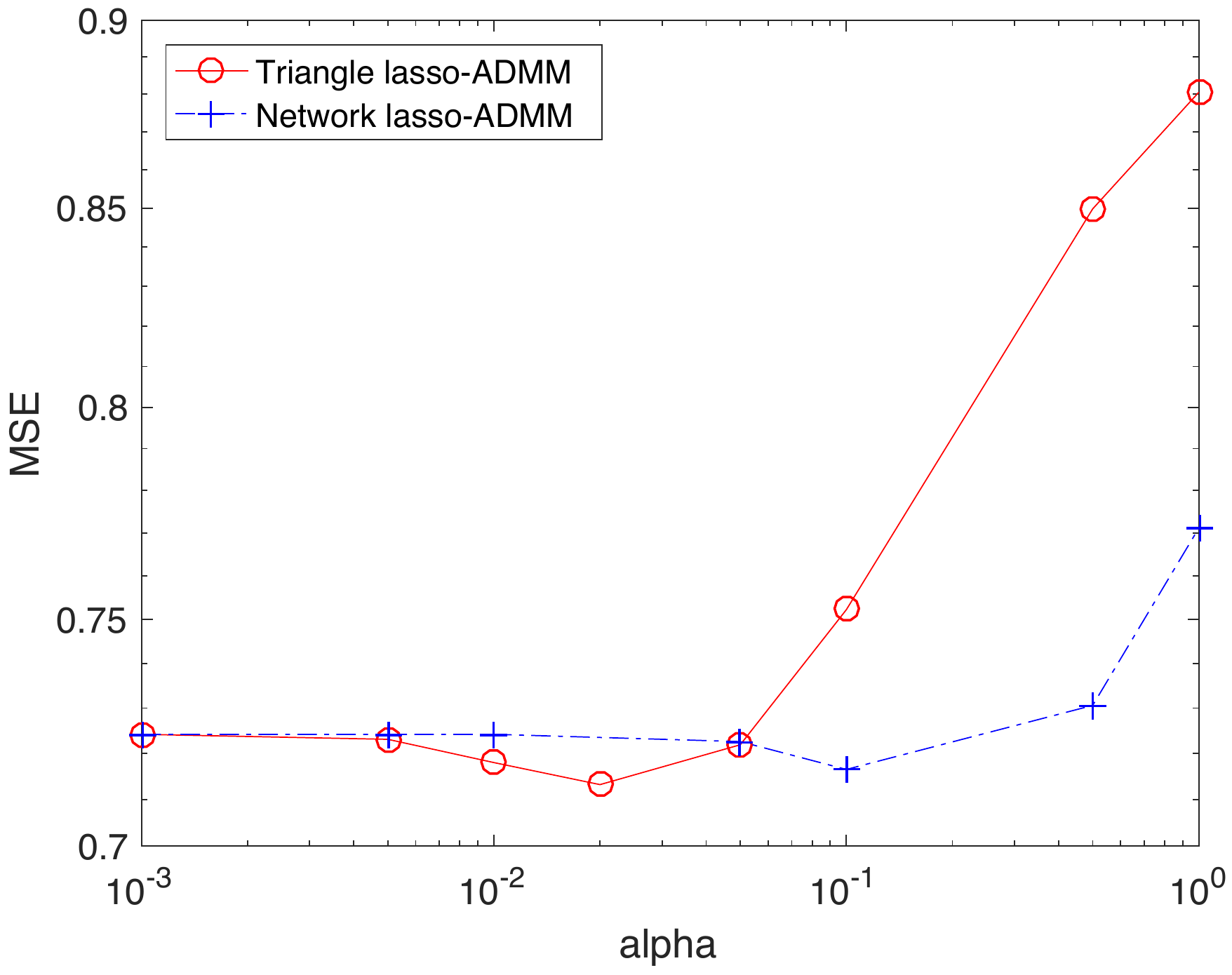}\label{figure_evalution_house_best_alpha}}
\subfigure[Robustness: low MSE]{\includegraphics[width=0.49\columnwidth]{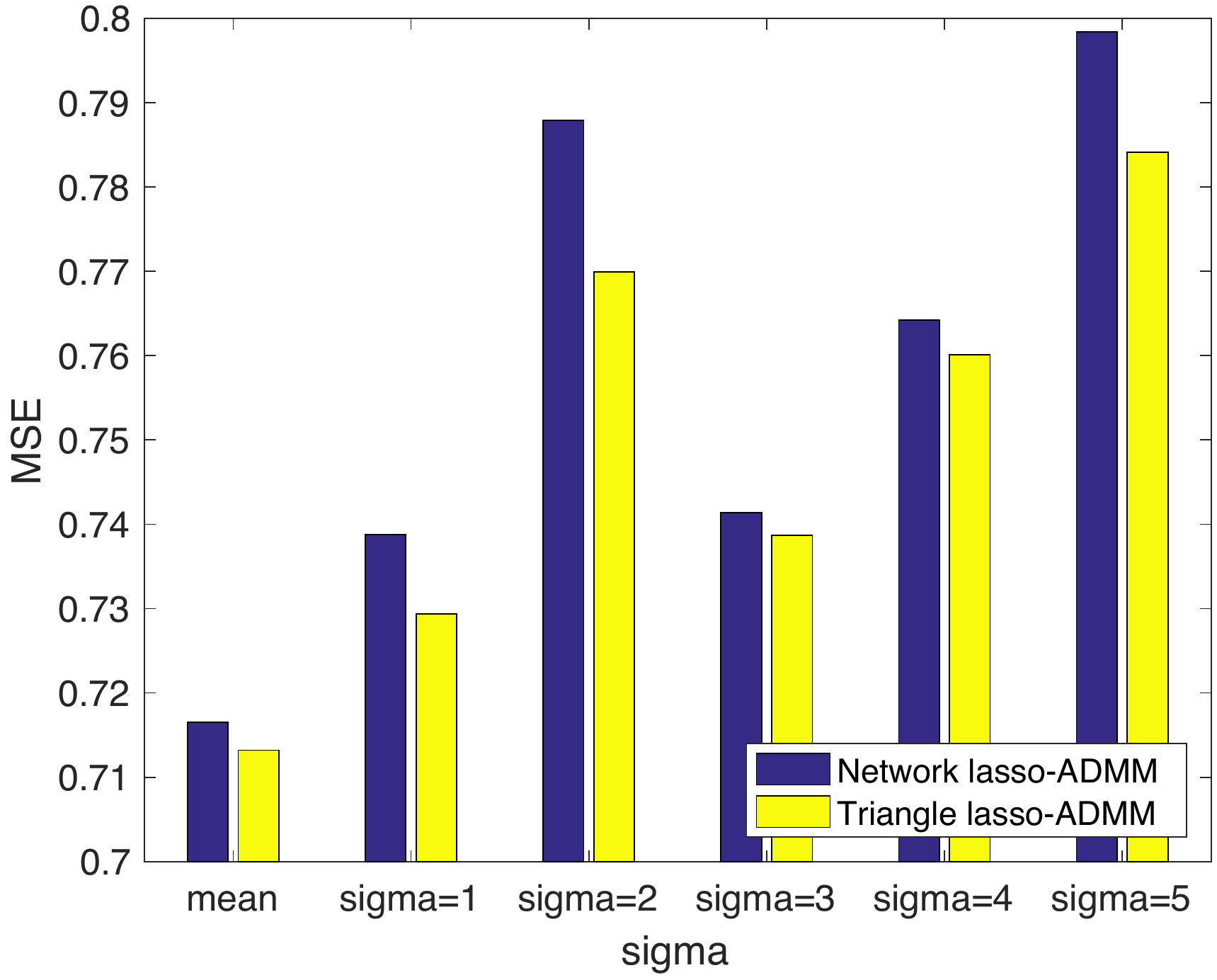}\label{figure_evalution_house_noise}}
\subfigure[Robustness: visualization]{\includegraphics[width=0.48\columnwidth]{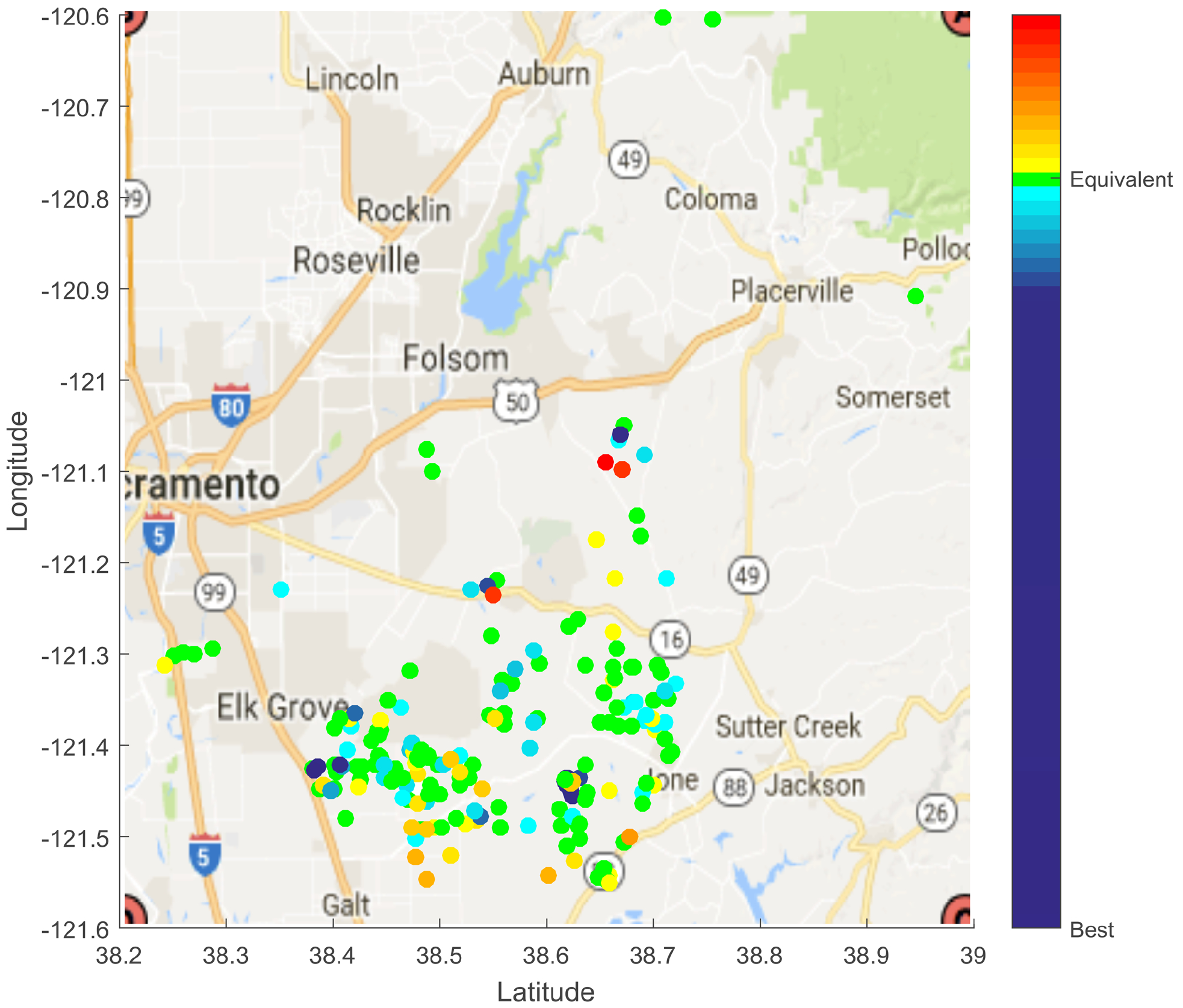}\label{figure_house_mse_tri}}
\subfigure[ Robustness: statistics]{\includegraphics[width=0.49\columnwidth]{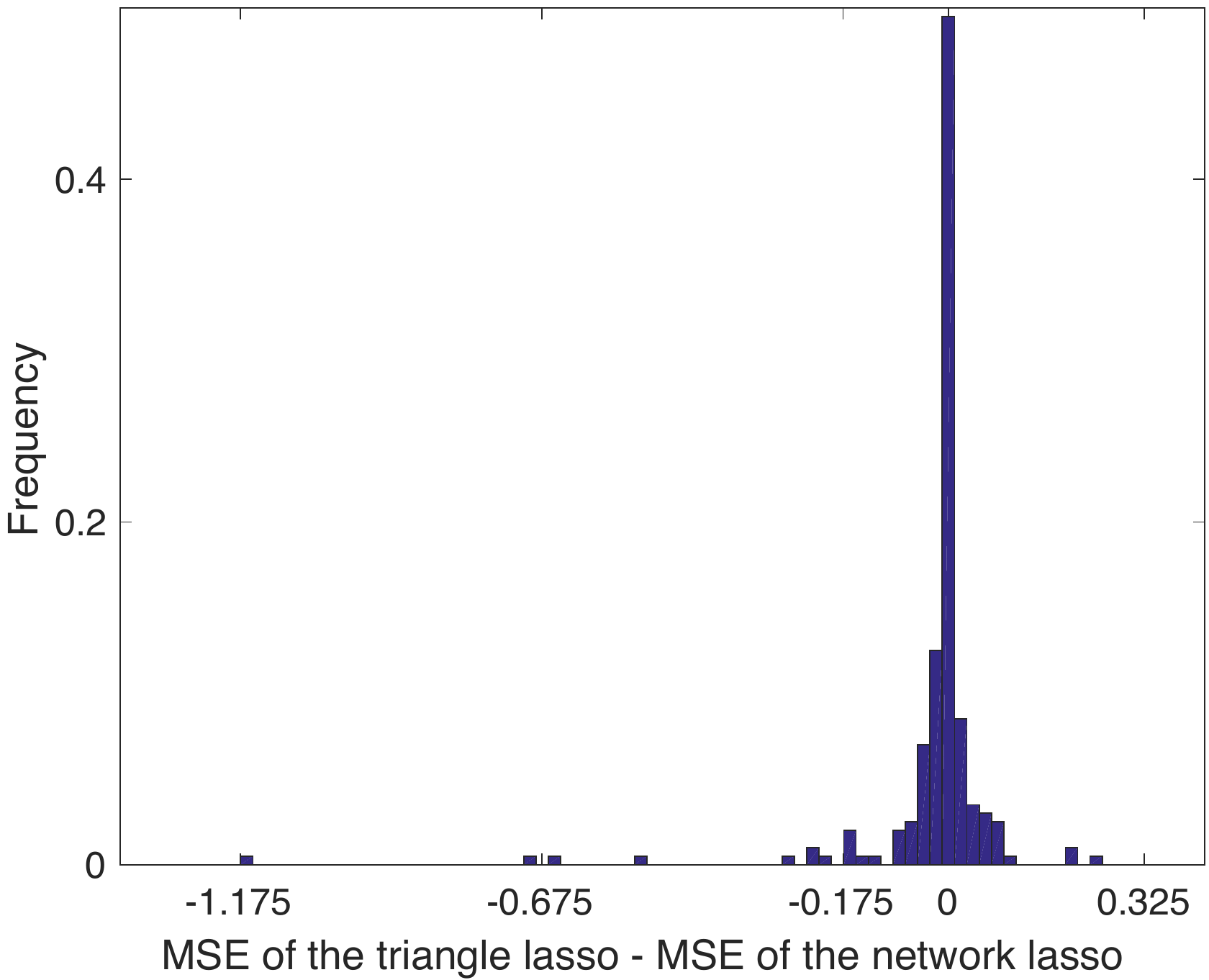}\label{figure_house_mse_diff_hist}}
\caption{The illustration of the best $\alpha$, and the comparison of the prediction accuracy with the best $\alpha$. Triangle lasso is more robust than the network lasso  because of the lower MSE.}
\label{figure_acc_effi_house}
\end{figure*}

\textbf{Graph construction and metrics.}   If the dataset is a graph dataset, triangle lasso use the graph directly. In some cases, if the dataset is not a graph,  the graph is usually generated by using the following rules in default:
\begin{itemize}
\item If the dataset has missing values, those missing values are filled by the mean values of the corresponding features.
\item Each instance is represented by a vertex.
\item Given any two arbitrary vertices, if one of them is the k-nearest peers ($k\ge 1$) of the other one, there is an edge between them.
\end{itemize}
Additionally, we evaluate the prediction accuracy by using the Mean Square Error (MSE). The small MSE leads to the highly accurate prediction. Given a dataset with imperfect data, if an algorithm yields smaller MSE than other algorithms, its prediction is thus more accurate than others. Therefore, it is more robust to the imperfect data than others.  We record the run time (seconds) to evaluate the efficiency.

\subsection{Prediction tasks}

\begin{table}[!]
\centering
\caption{Statistics of the datasets.}
\label{table_datasets}
\begin{tabular}{c|c|c|c}
\hline 
Datasets & Data size & Dimensions & Missing values\tabularnewline
\hline 
\hline 
RET & $985$ & $4$ & $17\%$\tabularnewline
\hline 
AOM & $126$ & $38$ & $17.1\%$\tabularnewline
\hline 
wiki4HE & $913$ & $13$ & $16.6\%$\tabularnewline
\hline 
DJI & $750$ & $4$ & $20\%$\tabularnewline
\hline 
cpusmall & $8192$ & $12$ & $20\%$\tabularnewline
\hline 
\end{tabular}
\end{table}

\textbf{Datasets. } The empirical studies are mainly conducted on the following four datasets. The statistics of those datasets are presented in Table \ref{table_datasets}.  It is worth noting that all of them contain many missing values. Those missing values are filled by using zeros in the raw datasets. In all experiments, the values of each feature  is standardized to zero mean and unit variance. We use $5$-fold cross validation to evaluate the robustness of triangle lasso. For each instance in the validation dataset, we find its nearest neighbour from the training dataset. Then, we use the weight of the nearest neighbour to conduct prediction and evaluate the robustness of the solutions. 

\begin{itemize}
\item \textbf{Real estate transactions (RET).} The dataset is the real estate transactions over a week period in May 2008 in the Greater Sacramento area \footnote{https://support.spatialkey.com/spatialkey-sample-csv-data}.   The latitude and longitude features of each house are used to construct the graph. Each house is profiled by using features: \textit{number of beds}, \textit{number of baths} and \textit{square feet}. The response is the \textit{sales price}. The task is to predict price of a house. $17\%$ of the house sales are missing at least one of the features. 

 \item \textbf{AusOpen-men-2013 (AOM).} A collection containing the match statistics for men at the Australian Open tennis tournaments of the year 2013\footnote{http://archive.ics.uci.edu/ml/datasets/Tennis+Major+Tourname\\ nt+Match+Statistics}.  Each instance has $38$ features, and the response is  \textit{Result}. The task is to predict  the winner for two tennis players. This data matrix contains $17.1\%$ missing values. 
 
\item \textbf{wiki4HE.} Survey of faculty members from two Spanish universities on teaching uses of Wikipedia\footnote{http://archive.ics.uci.edu/ml/datasets/wiki4HE} \cite{Meseguer2016Factors}. We pick the first question and its answer from each module, and finally obtain $13$ features, namely \textit{PU1}, \textit{PEU1}, \textit{ENJ1}, \textit{QU1}, \textit{VIS1}, \textit{IM1}, \textit{SA1}, \textit{USE1}, \textit{PF1}, \textit{JR1}, \textit{BI1}, \textit{INC1}, and \textit{EXP1}. The response is \textit{USERWIKI}. The task is to predict whether a teacher register an account in wikipedia site. The data matrix contains $16.6\%$ missing values.

\item \textbf{Dow Jones Index (DJI).} This dataset contains weekly data for the Dow Jones Industrial Index \footnote{http://archive.ics.uci.edu/ml/datasets/Dow+Jones+Index}. Each instance is profiled by using features: \textit{open} (price), \textit{close} (price) and \textit{volume}. The response is \textit{next\_week\_open} (price). The task is to predict the open price in the next week. The raw dataset does not contain imperfect data. We thus randomly pick $20\%$ values in the data matrix, and set them by using zeros. 

\item \textbf{cpusmall.} This dataset is a collection of a computer systems activity measures, which is obtained from LIBSVM website \footnote{https://www.csie.ntu.edu.tw/$\sim$ cjlin/libsvmtools/datasets/regre\\ ssion.html\#cpusmall}. Each instance is profiled by $12$ features. The task is to predict the portion of time (\%) that cpus run in user mode. In the experiment, we randomly pick $20\%$ values in the data matrix, and fill those values to be zeros as the imperfect data. 
\end{itemize}   

\begin{figure}[t]
\setlength{\abovecaptionskip}{0pt}
\setlength{\belowcaptionskip}{0pt}
\centering 
\subfigure[Efficiency]{\includegraphics[width=0.48\columnwidth]{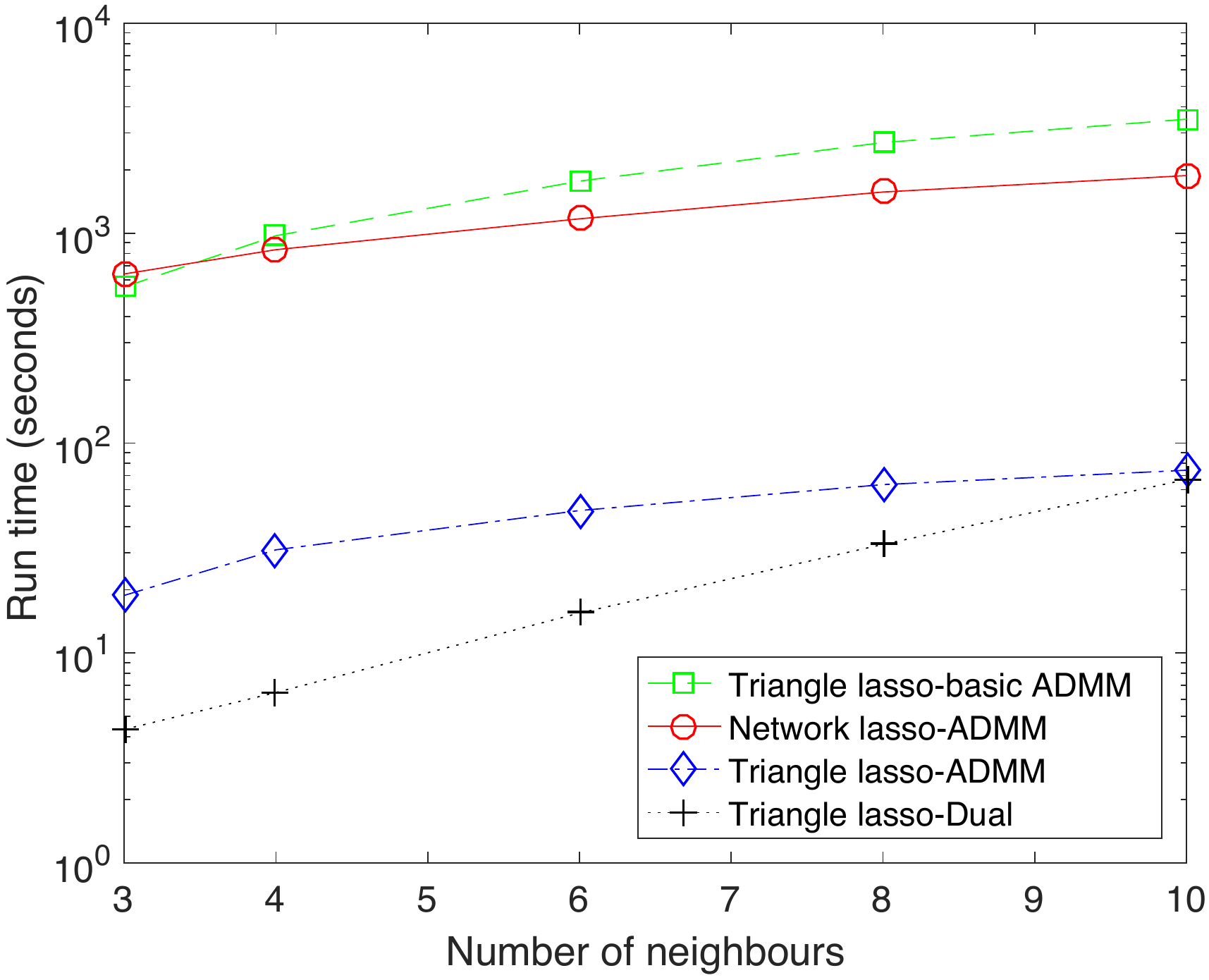}\label{figure_evalution_house_runtime}}
\subfigure[Efficiency]{\includegraphics[width=0.49\columnwidth]{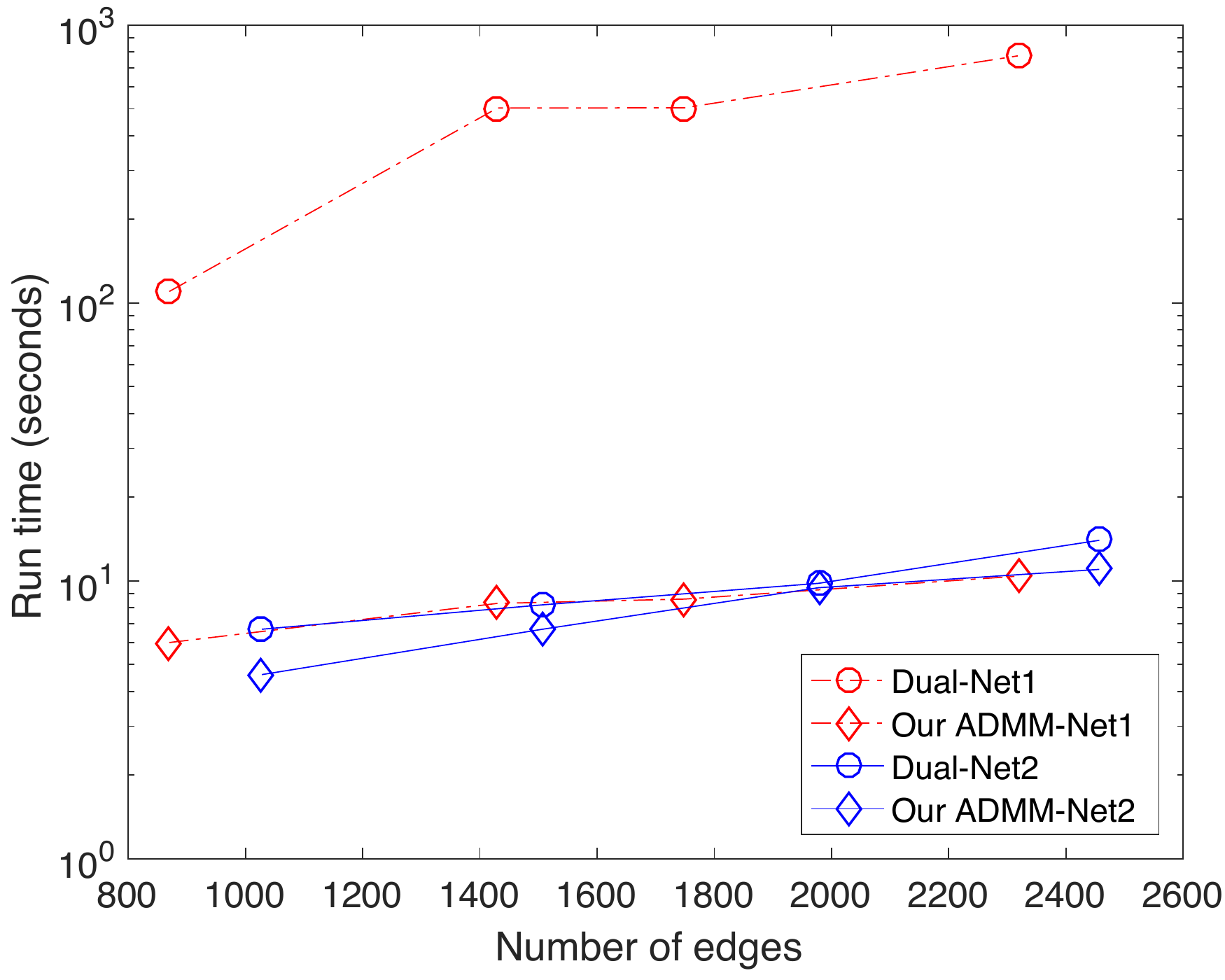}\label{figure_evalution_runtime_network_types}}
\label{figure_house_solution_quality}
\caption{The comparison of the efficiency. The Dual method is more efficient than the  ADMM method in a sparse graph but less efficient in a dense graph. }
\end{figure}

\begin{figure*}[t]
\setlength{\abovecaptionskip}{0pt}
\setlength{\belowcaptionskip}{0pt}
\centering 
\subfigure[AOM, Robustness: low MSE]{\includegraphics[width=0.49\columnwidth]{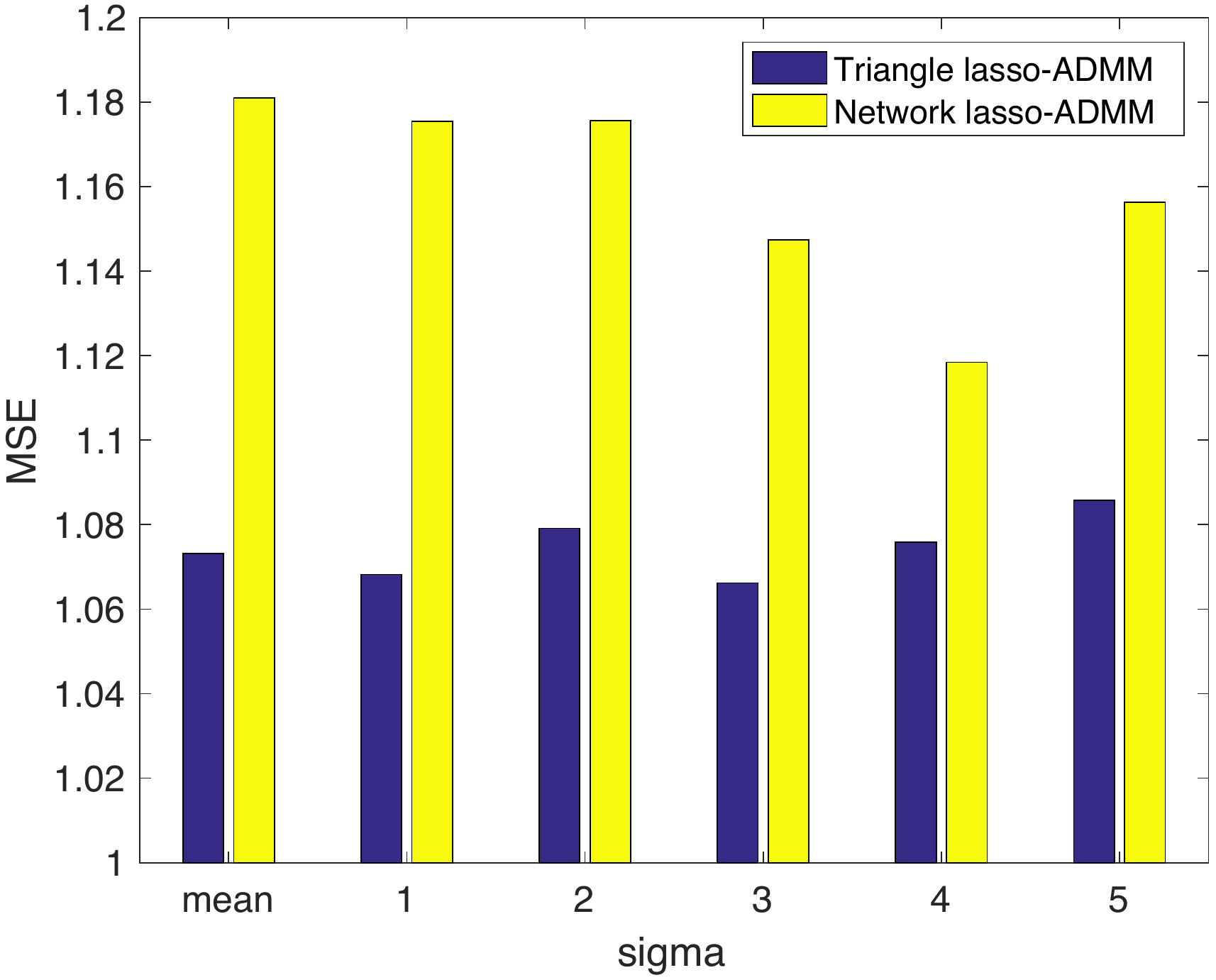}\label{figure_robust_aom}}
\subfigure[wiki4HE, Robustness: low MSE]{\includegraphics[width=0.49\columnwidth]{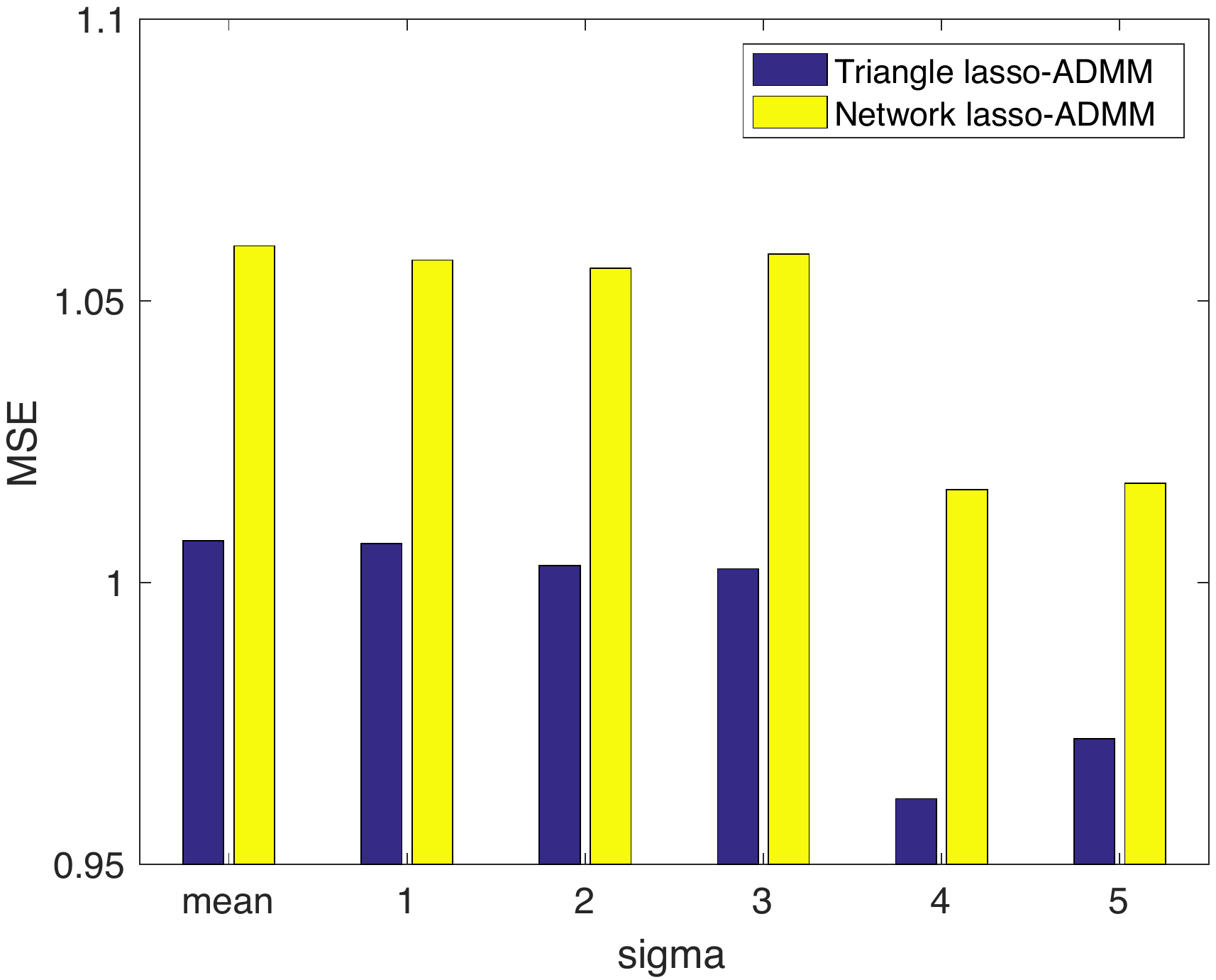}\label{figure_robust_wiki4h}}
\subfigure[DJI, Robustness: low MSE]{\includegraphics[width=0.49\columnwidth]{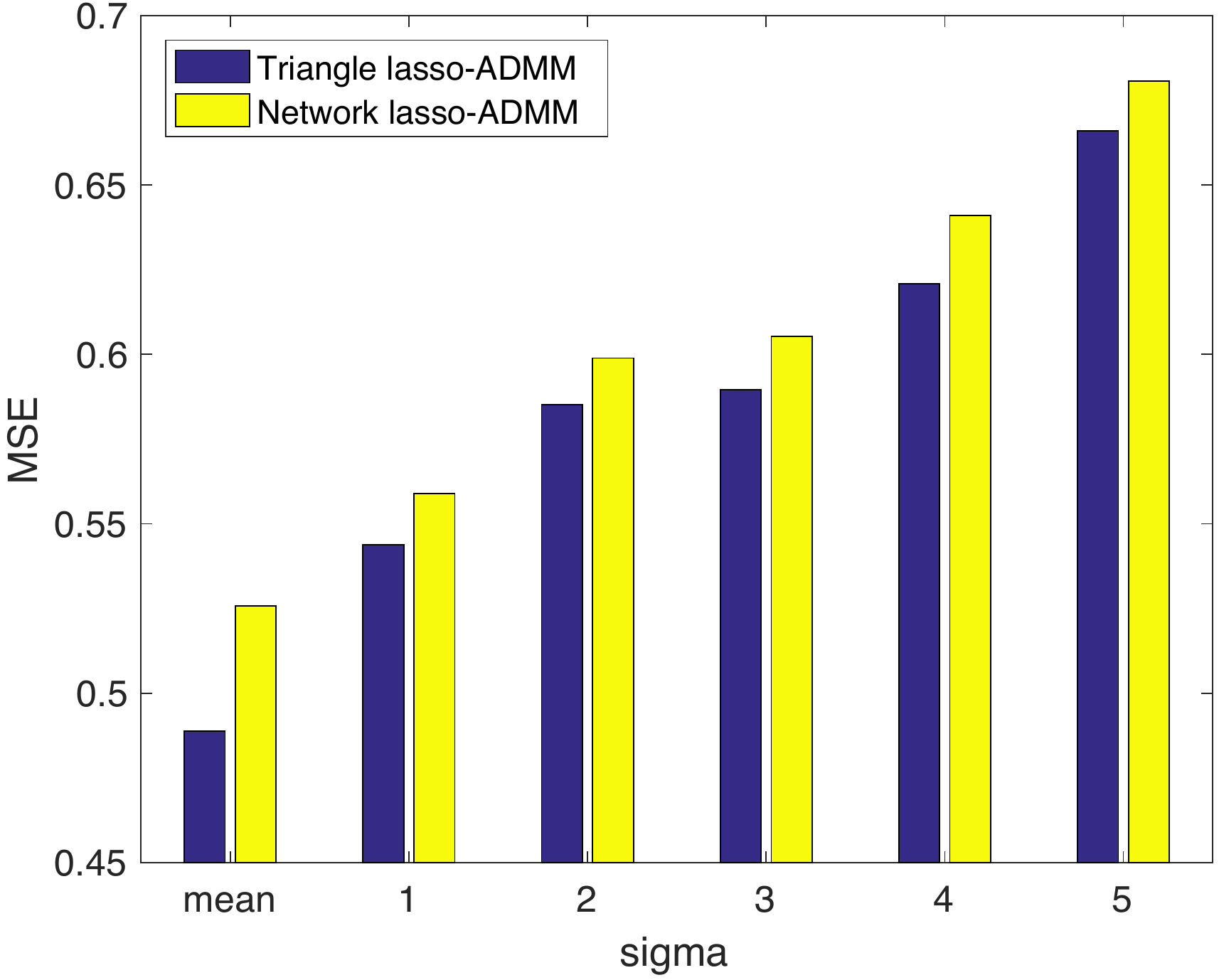}\label{figure_robust_dji}}
\subfigure[cpusmall, Robustness: low MSE]{\includegraphics[width=0.495\columnwidth]{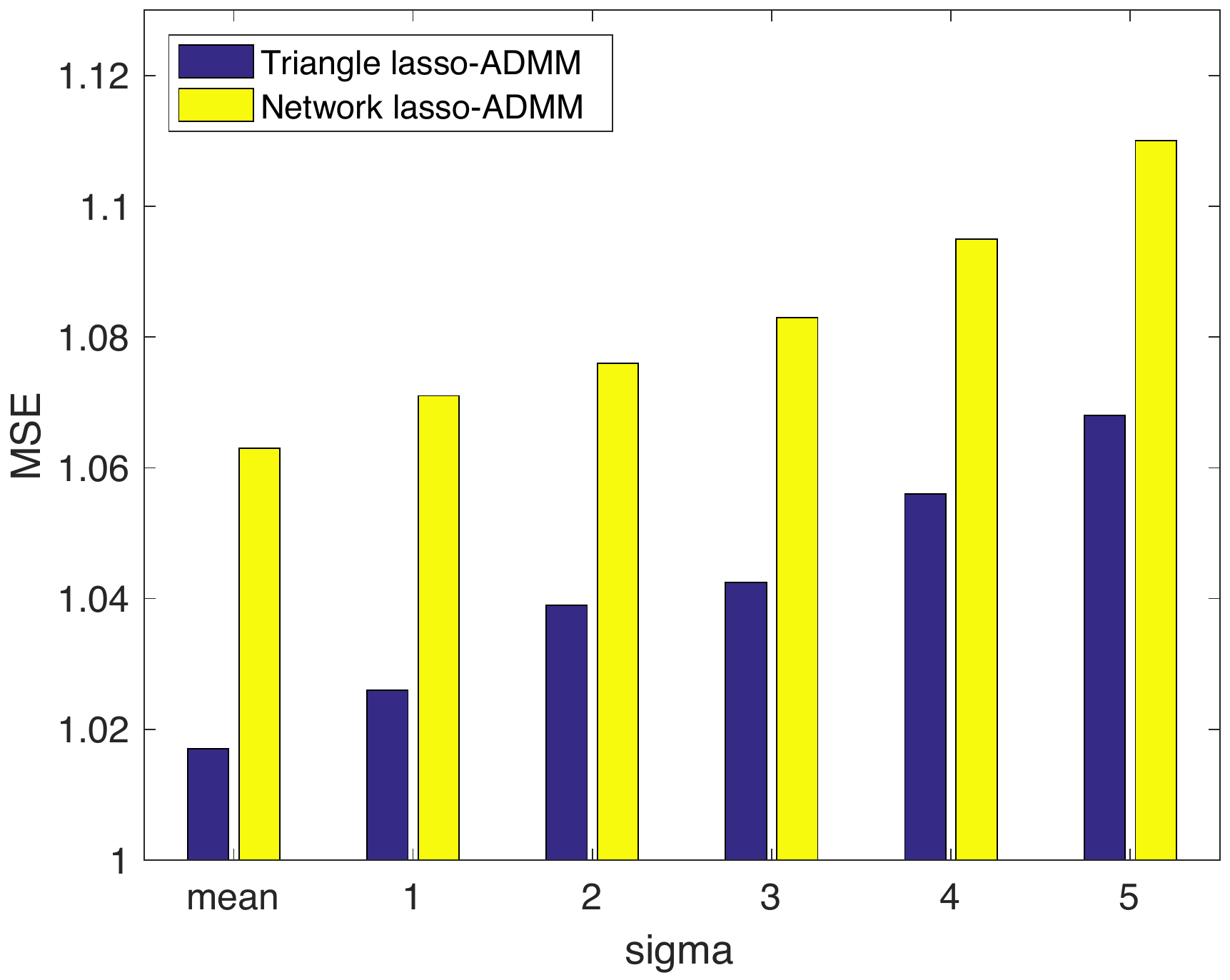}\label{figure_robust_cpusmall}}
\caption{The comparison of the MSE by varying the imperfect data. It shows that triangle lasso is more robust than network lasso because of its low MSE.}
\label{figure_prediction_robust}
\end{figure*}

\begin{table*}[!]
\centering
\caption{CPU seconds consumed when handling AOM, wiki4HE, DJI and cpusmall}
\label{table_efficiency}
\begin{tabular}{c|c|c|c|c|c}
\hline 
\multicolumn{2}{c|}{Algorithms} & basic ADMM & Network lasso & Triangle lasso-ADMM & Triangle lasso-Dual\tabularnewline
\hline 
\hline 
\multirow{6}{*}{AOM} & $\alpha=100$, $NoN=3$ & $576$ & $397$ & $23$ & $\mathbf{12}$\tabularnewline
\cline{2-6} 
 & $\alpha=100$, $NoN=4$ & $880$ & $607$ & $32$ & $\mathbf{29}$\tabularnewline
\cline{2-6} 
 & $\alpha=100$, $NoN=5$ & $1100$ & $824$ & $\mathbf{37}$ & $84$\tabularnewline
\cline{2-6} 
 & $\alpha=200$, $NoN=5$ & $1124$ & $699$ & $\mathbf{43}$ & $76$\tabularnewline
\cline{2-6} 
 & $\alpha=300$, $NoN=5$ & $1211$ & $700$ & $\mathbf{43}$ & $79$\tabularnewline
\cline{2-6} 
 & $\alpha=500$, $NoN=5$ & $1547$ & $748$ & $\mathbf{56}$ & $81$\tabularnewline
\hline 
\multirow{4}{*}{wiki4HE} & $\alpha=100$, $NoN=3$ & out of memory & out of memory & $1864$ & $\mathbf{337}$\tabularnewline
\cline{2-6} 
 & $\alpha = 100$, $NoN=4$ & out of memory & out of memory & $\mathbf{593}$ & $2020$\tabularnewline
\cline{2-6} 
 & $\alpha = 50$, $NoN=3$ & out of memory & out of memory & $1305$ & $\mathbf{368}$\tabularnewline
\cline{2-6} 
 & $\alpha=200$, $NoN=3$ & out of memory & out of memory & $2854$ & $\mathbf{341}$\tabularnewline
\hline 
\multirow{4}{*}{DJI} & $\alpha=100$, $NoN=5$ & $132$ & $26$ & $\mathbf{3}$ & $6$\tabularnewline
\cline{2-6} 
 & $\alpha=50$, $NoN=5$ & $99$ & $54$ & $\mathbf{2}$ & $31$\tabularnewline
\cline{2-6} 
 & $\alpha=100$, $NoN=10$ & $499$ & $54$ & $\mathbf{12}$ & $32$\tabularnewline
\cline{2-6} 
 & $\alpha=1$, $NoN=10$ & $100$ & $54$ & $\mathbf{2}$ & $31$\tabularnewline
\hline 
\multirow{4}{*}{cpusmall} & $\alpha=100$, $NoN=2$ & out of memory & out of memory & $7248$ & $\mathbf{1175}$\tabularnewline 
\cline{2-6} 
& $\alpha=200$, $NoN=3$ & out of memory & out of memory & $8940$ & $\mathbf{5600}$\tabularnewline
\cline{2-6} 
& $\alpha = 200$, $NoN=4$ & out of memory & out of memory & $\mathbf{10276}$ & $15940$\tabularnewline
\cline{2-6}
& $\alpha = 500$, $NoN=4$ & out of memory & out of memory & $\mathbf{11745}$ & $16874$\tabularnewline
\hline 
\end{tabular}
\end{table*}

\textbf{Results for RET.}
Each row of $X$ represents the weights of a house.  $B$ represents the prices of houses.     Additionally,  each vertex is connected with its $10$ nearest neighbours via  the latitude and longitude, and there exist $28521$ triangles in the network.

First, we need to find the best $\alpha$.   As illustrated in Fig. \ref{figure_evalution_house_best_alpha}, the comparison of MSE is conducted by varying $\alpha$.  With the increase of $\alpha$, the houses located in a same region begin to use the similar or identical weights.  The quality of the predictions is thus improved. When $\alpha$ is too large, the houses located in different regions also use the similar weights, resulting in the decrease of the accuracy of the prediction. In the experiment, we set $\alpha\mathrm{=}0.02$ for the triangle lasso, and set $\alpha\mathrm{=}0.1$ for the network lasso.

Second, we evaluate the robustness of the triangle lasso by  varying the missing values. The missing values are filled by using the data generated from a Gauss distribution whose mean is $\mu = 0$, and standard deviation $\sigma$ is varied from $1$ to $5$. `mean' represents those missing values are filled by zeros. As illustrated in Fig. \ref{figure_evalution_house_noise}, the triangle lasso yields a better prediction than the network lasso. The reason is that we use the shared neighbouring relation to decrease the impact of the missing values.  We present the predictions in Fig. \ref{figure_house_mse_tri}. The blue or red markers on the map represent the better or the worse predictions yielded by the triangle lasso, respectively. The distribution of those predictions are presented in Fig. \ref{figure_house_mse_diff_hist}. We find that many predictions are comparable for the triangle lasso and the network lasso, but the triangle lasso yields more better predictions than the worse predictions.

\begin{figure*}[!ht]
\setlength{\abovecaptionskip}{0pt}
\setlength{\belowcaptionskip}{0pt}
\centering 
\subfigure[Triangle lasso, $698$ seconds]{\includegraphics[width=0.48\columnwidth]{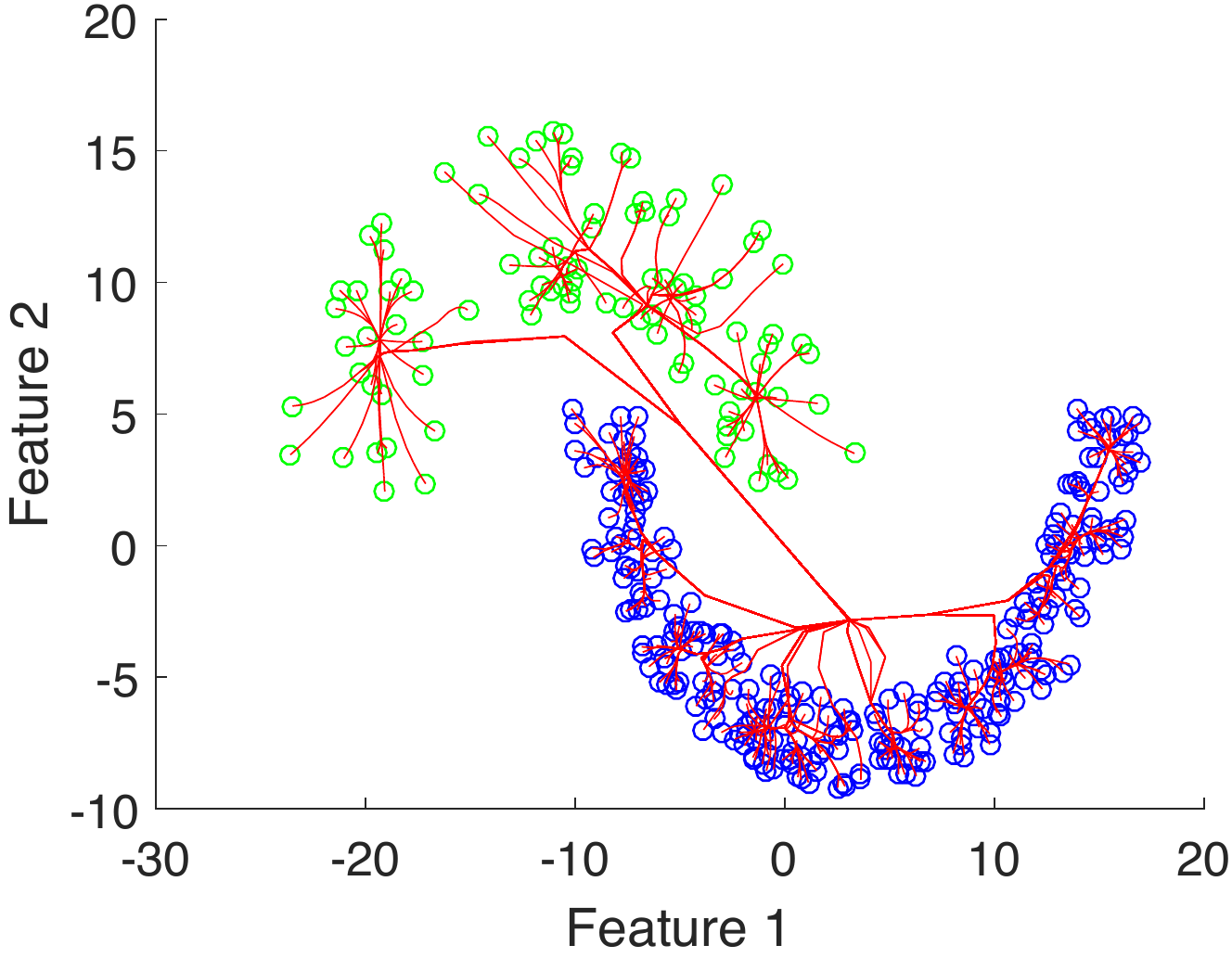}\label{figure_moon_tri_sigma0}}
\subfigure[AMA, $695$ seconds]{\includegraphics[width=0.48\columnwidth]{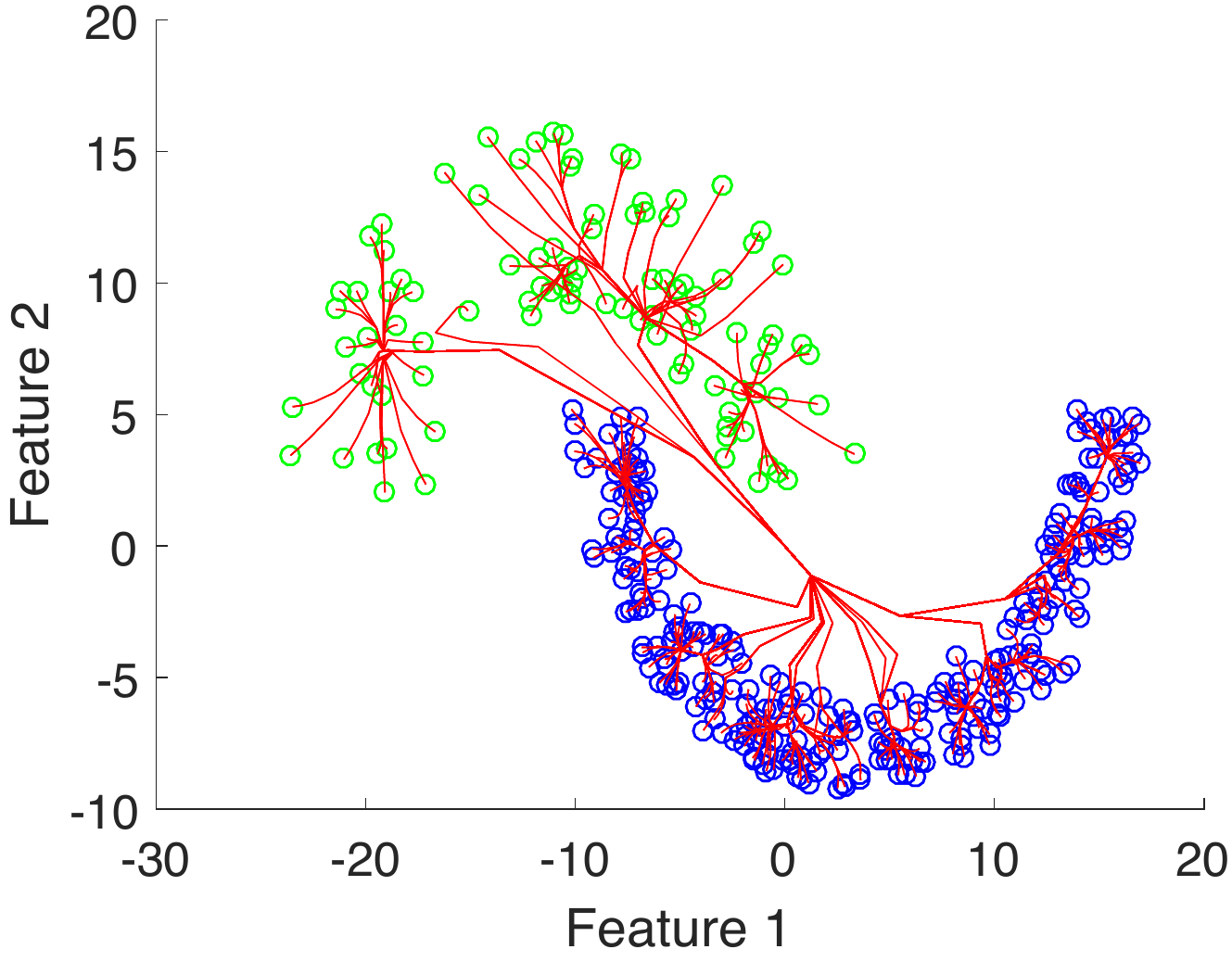}\label{figure_moon_net_sigma0}}
\subfigure[Triangle lasso,  $1.87$ seconds ]{\includegraphics[width=0.49\columnwidth]{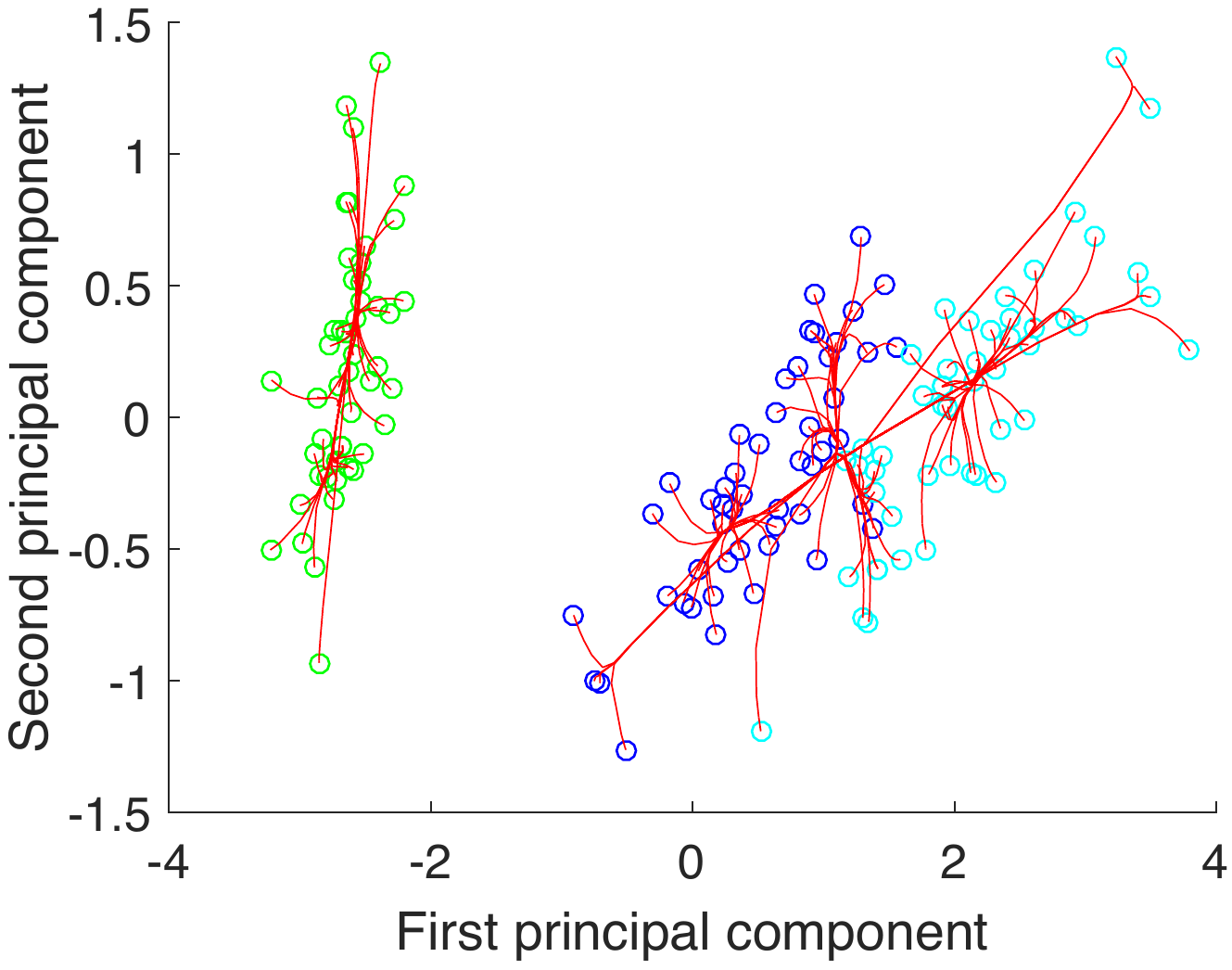}\label{figure_iris_tri}}
\subfigure[AMA, $0.96$ seconds]{\includegraphics[width=0.49\columnwidth]{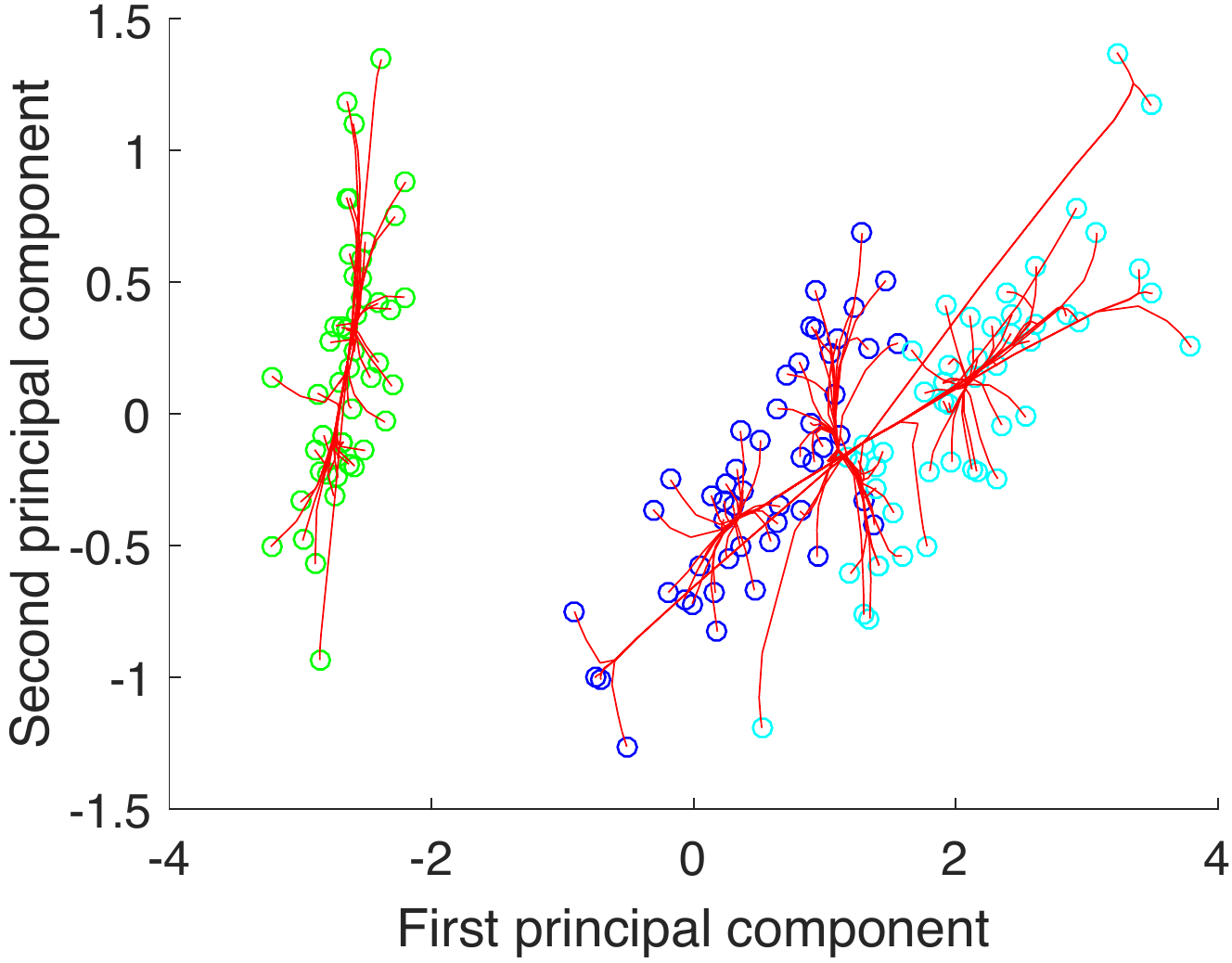}\label{figure_iris_net}}
\caption{The comparison of the cluster paths.}
\label{figure_clusterpath}
\end{figure*}
\begin{figure*}[!ht]
\setlength{\abovecaptionskip}{0pt}
\setlength{\belowcaptionskip}{0pt}
\centering 
\subfigure[Random \& small world networks]{\includegraphics[width=0.48\columnwidth]{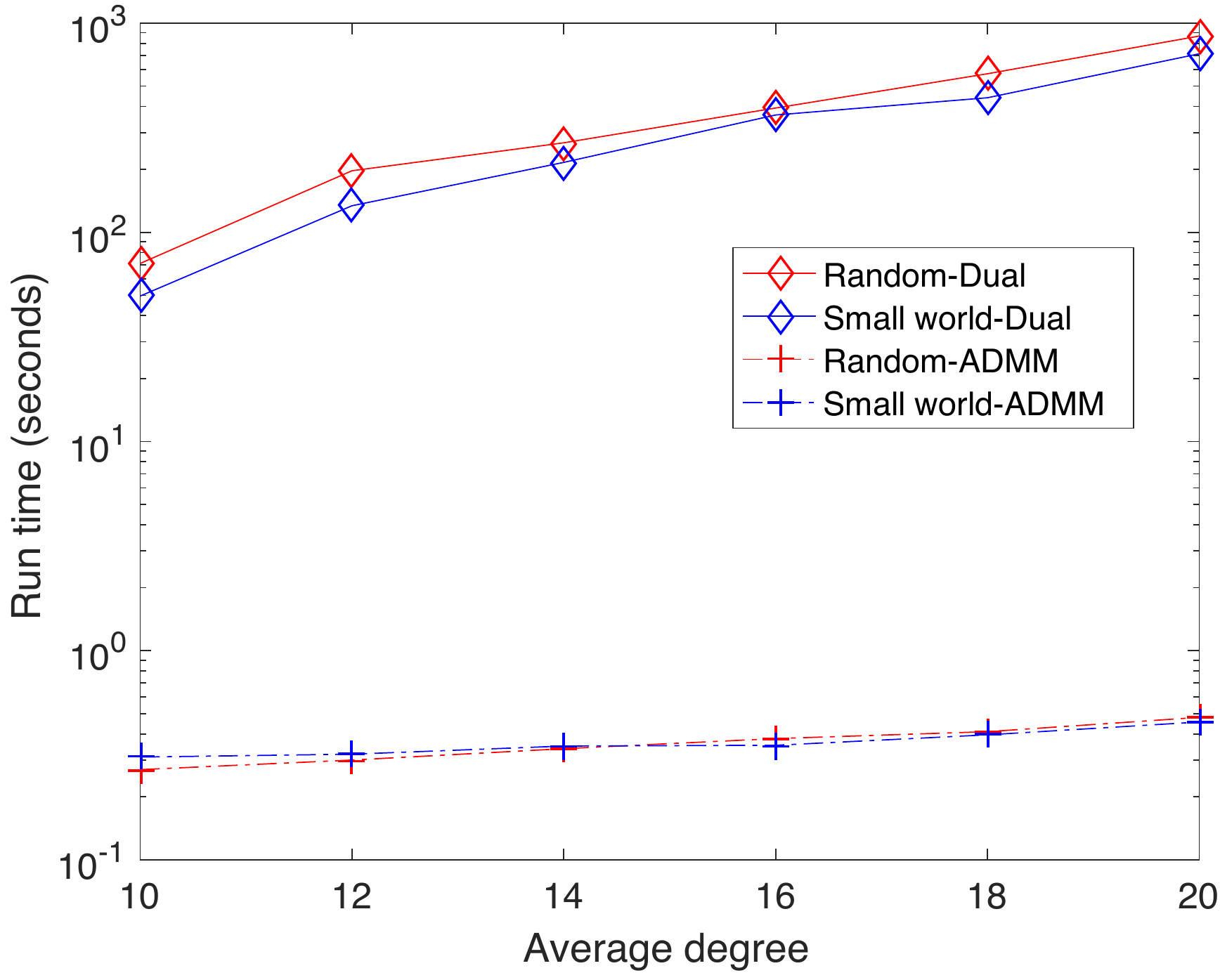}\label{figure_evaluation_runtime_random_small_world}}
\subfigure[Scale free network]{\includegraphics[width=0.48\columnwidth]{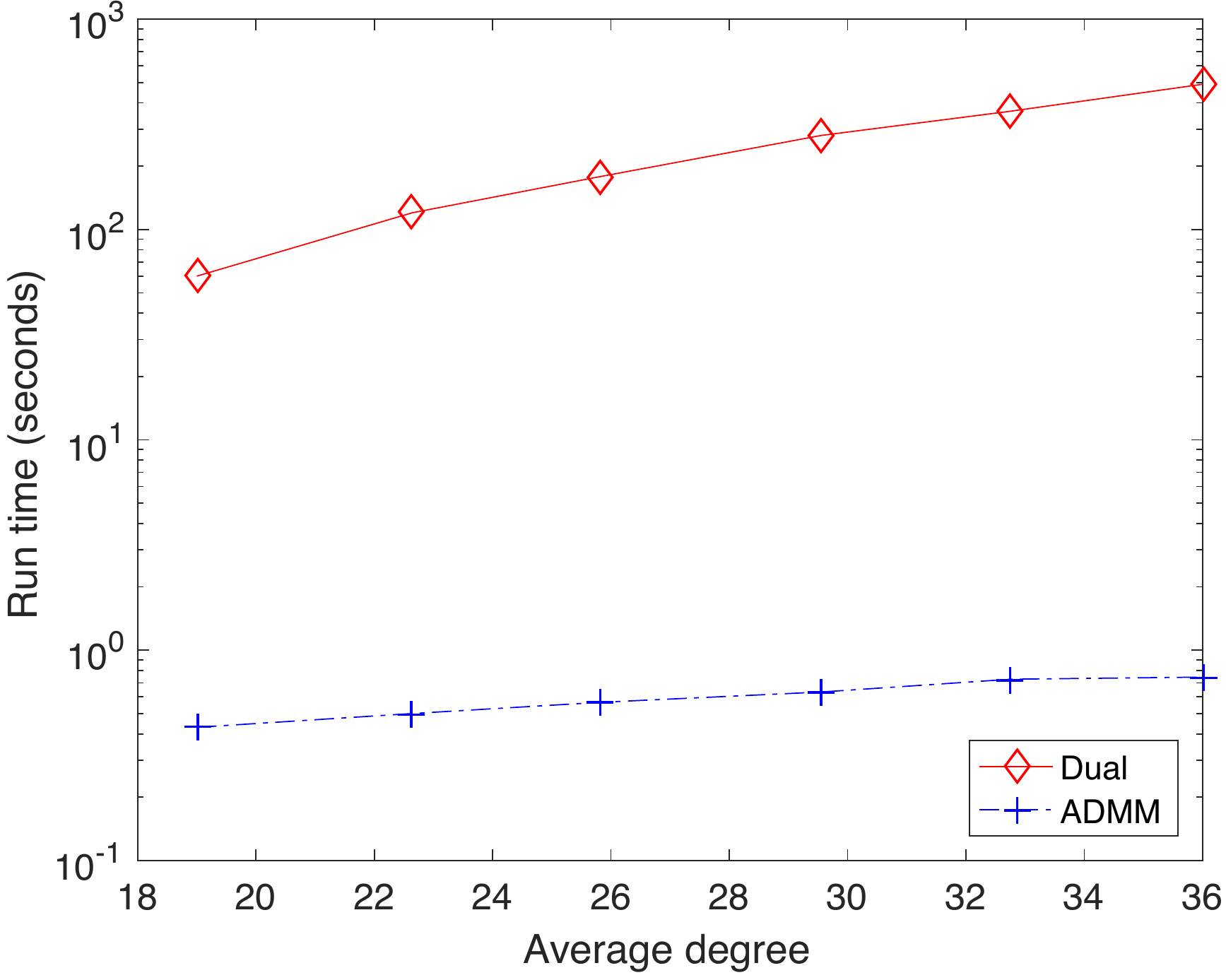}\label{figure_evaluation_runtime_scale_free}}
\subfigure[C3 network]{\includegraphics[width=0.48\columnwidth]{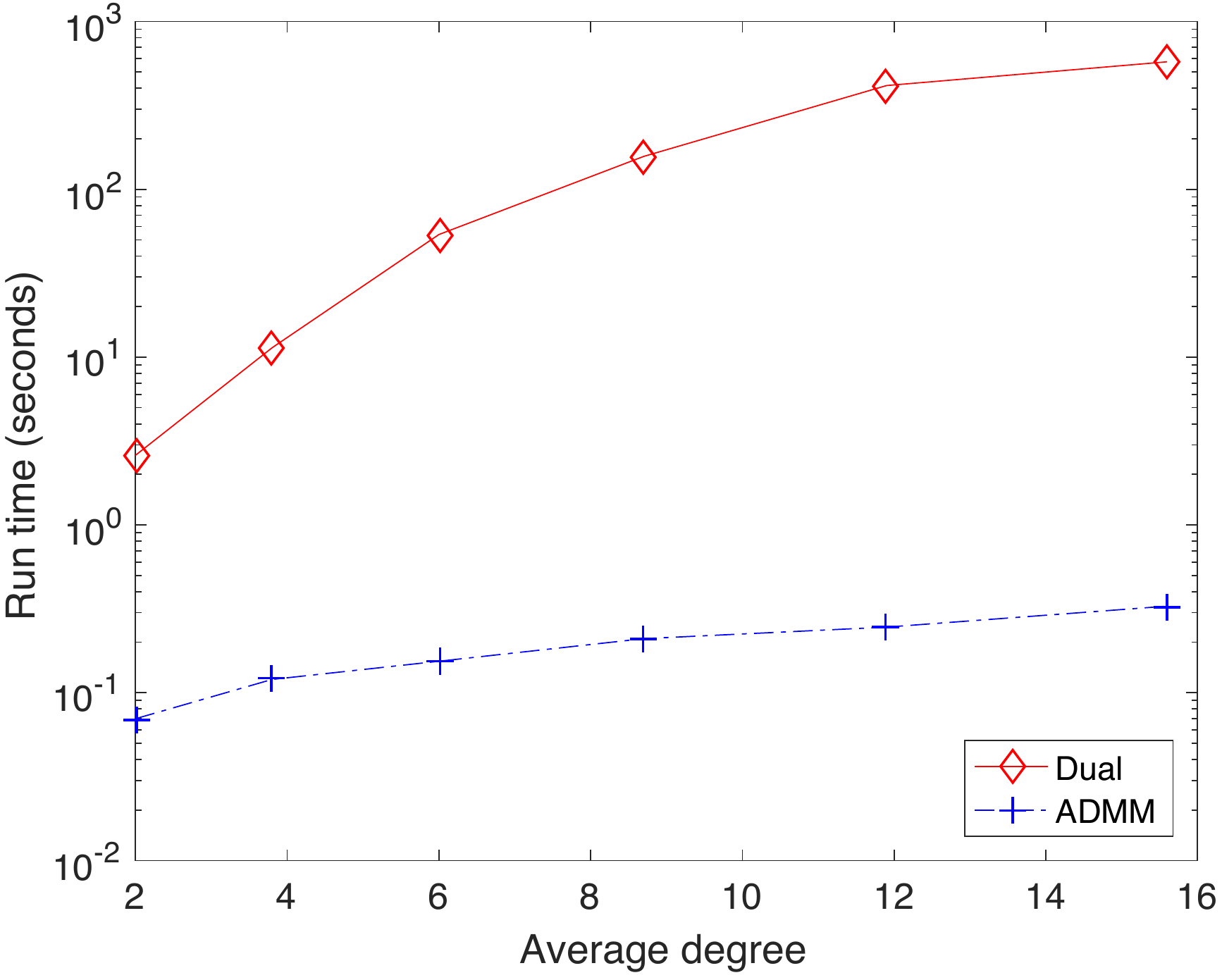}\label{figure_evaluation_runtime_partial_dense}}
\subfigure[ADMM v.s. Dual]{\includegraphics[width=0.495\columnwidth]{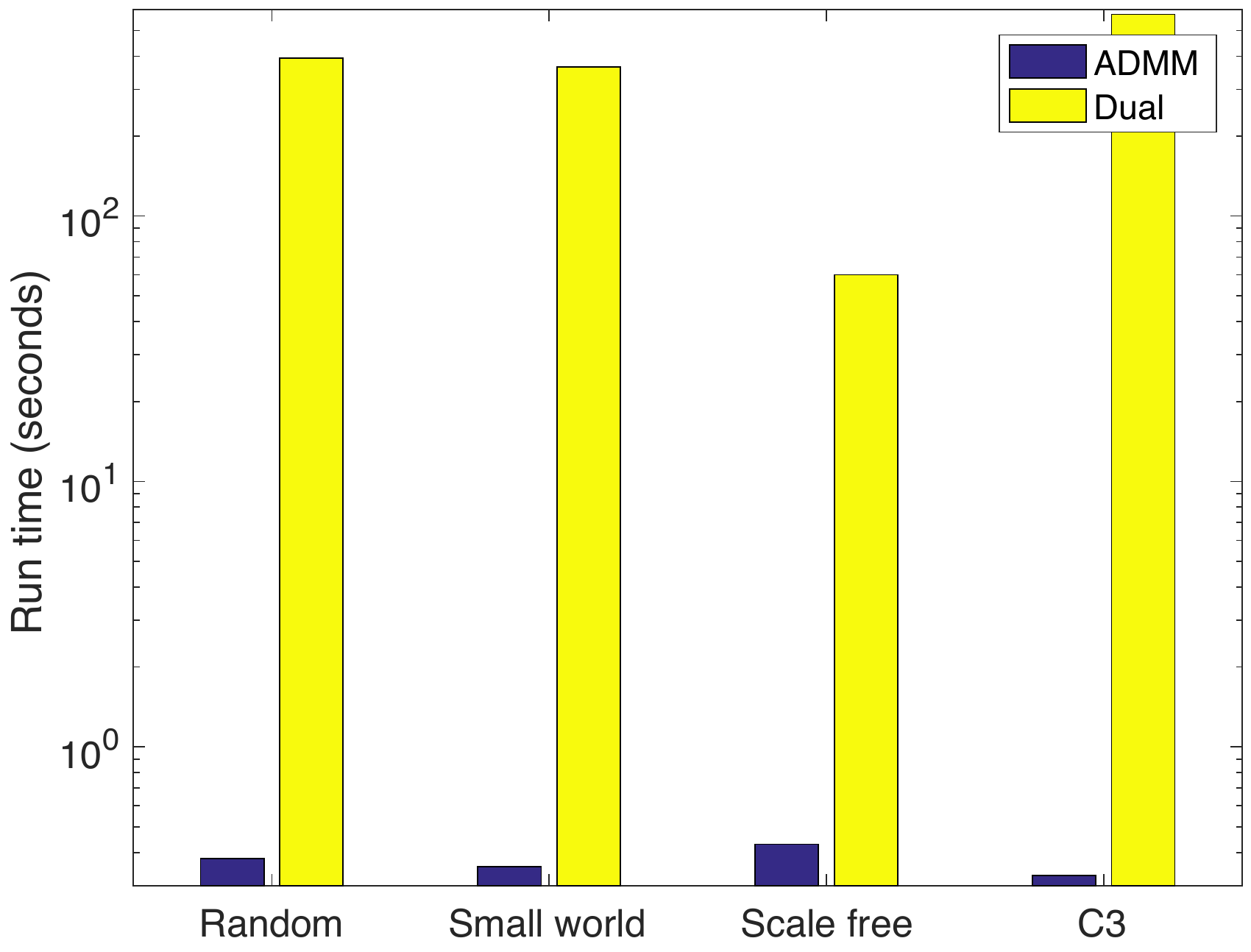} \label{figure_evaluation_runtime_different_networks_admm_dual}}
\caption{The comparison of the run time by varying the average degree in the first three subfigures, and fixing it to be $16$ in the last subfigure.}
\label{figure_effi_networks}
\end{figure*}

Third, we evaluate the efficiency of our methods by varying the number of neighbours for each vertex. Our method, which yields the accurate solution, is denoted by $\emph{Dual}$. As illustrated in Fig. \ref{figure_evalution_house_runtime}, our methods are more efficient than the basic ADMM and the network lasso. Meanwhile, it can be seen that the Dual method is more efficient than our ADMM method when the number of neighbours is not large ($\le 10$). But, the superiority is decreased with the increase of the number of the neighbours. When we use the ridge regression model, the Dual method yields $X_\ast$ by solving a large number of linear equations, which is time-consuming for a large and dense graph.   Additionally, we build the network in a different way. We set a threshold $10km$, and connect the neighbouring vertices whose distance is less than the threshold.  The network, which is yielded by using the neighbours of vertices, is denoted by \emph{Net1}. Similarly, the network, which is yielded by using the threshold, is denoted by \emph{Net2}. We evaluate the efficiency of our methods by varying the number of edges in those networks.  As illustrated in Fig. \ref{figure_evalution_runtime_network_types}, both methods perform better in the Net1 than in the Net2. The efficiency of the Dual method is decreased sharply in the Net2.  The reason is that the vertices in the Net2 tend to completely connect with their neighbours. Thus, their weights are highly non-separable with the weights of their nighbours, which decreases  the efficiency of the Dual method sharply.

\textbf{Results for AOM, wiki4HE, DJI and cpusmall.} First, we evaluate the robustness of triangle lasso. The missing values are still filled by using the data generated from a Gauss distribution whose mean is $\mu = 0$, and standard deviation $\sigma$ is varied from $1$ to $5$.  `mean' represents those missing values are filled by zeros.  As illustrated in Figure \ref{figure_prediction_robust}, triangle lasso yields a lower MSE than network lasso in each experiment, which shows the robustness of triangle lasso\footnote{It is out of memory for network lasso to handle wiki4HE. We degenerate triangle lasso to the settings of network lasso, and use our methods to obtain the MSE corresponding to network lasso.}. Second, we test the efficiency of our methods. \textit{NoN} represents the number of neighbours for each vertex.  As shown in Table \ref{table_efficiency}, our methods are more efficient than their counterparts. Note that the ADMM method outperforms the Dual method when the graph is dense. That is, the average of the number of neighbours, i.e. $NoN$ is relatively large. 
When a graph is dense, the optimization variables are highly non-separable. In the case, it is time-consuming to be solved.  Since the Dual method want to return a highly accurate solution, it needs more time than the ADMM. But, when the graph is sparse, many of the optimization variables are separable, which makes the optimization problem easy to be solved. In the case, the Dual method is performed efficiently, and thus outperforms the ADMM method.

\subsection{Convex clustering.}
\textbf{Dataset and settings.} We aim to obtaining the cluster paths on the moon dataset \footnote{https://cs.joensuu.fi/sipu/datasets/jain.txt} and the iris dataset \footnote{http://archive.ics.uci.edu/ml/datasets/Iris}. The moon dataset contains $373$ instances, and each instance has two features. The iris dataset contains $150$ instances, and each instance has four features.  In the experiment, the values in each feature  is standardized to zero mean and unit variance. We run our ADMM method to obtain a cluster path, and compare it with the state-of-the-art convex clustering method, i.e. AMA. Consider the moon dataset. The network is built by connecting a vertex with its $50$ nearest neighbours. The initial $\alpha$ is $100$, and it is increased by multiplying a step size. The step size is initialized to be $1$ and is increased by $2$ at each iteration.  Similarly, consider the iris dataset. The network is built by connecting a vertex with its $10$ nearest neighbours. The initial $\alpha$ is $0.01$, and it is increased by multiplying a step size. The step size is initialized to be $1$ and is increased by $0.1$ at each iteration.  In order to draw the cluster path for the iris dataset,  we pick and visualize the first and second principal components  by using the Principal Component Analysis (PCA) method. 
%

\textbf{Results.}  As illustrated in Fig. \ref{figure_clusterpath},  the triangle lasso yields sightly better cluster paths than AMA with the comparable efficiency. The reason is that triangle lasso uses the neighbouring information for the vertices to find the cluster membership. In the network science, the neighbours are usually viewed as  the most valuable information for a vertex \cite{Prell:2011:SNA:2222515}. Therefore, triangle lasso outperforms AMA, and yields a better cluster path.

\begin{figure*}[t]
\setlength{\abovecaptionskip}{0pt}
\setlength{\belowcaptionskip}{0pt}
\centering 
\subfigure[$\alpha=1$, $37$ seconds]{\includegraphics[width=0.46\columnwidth]{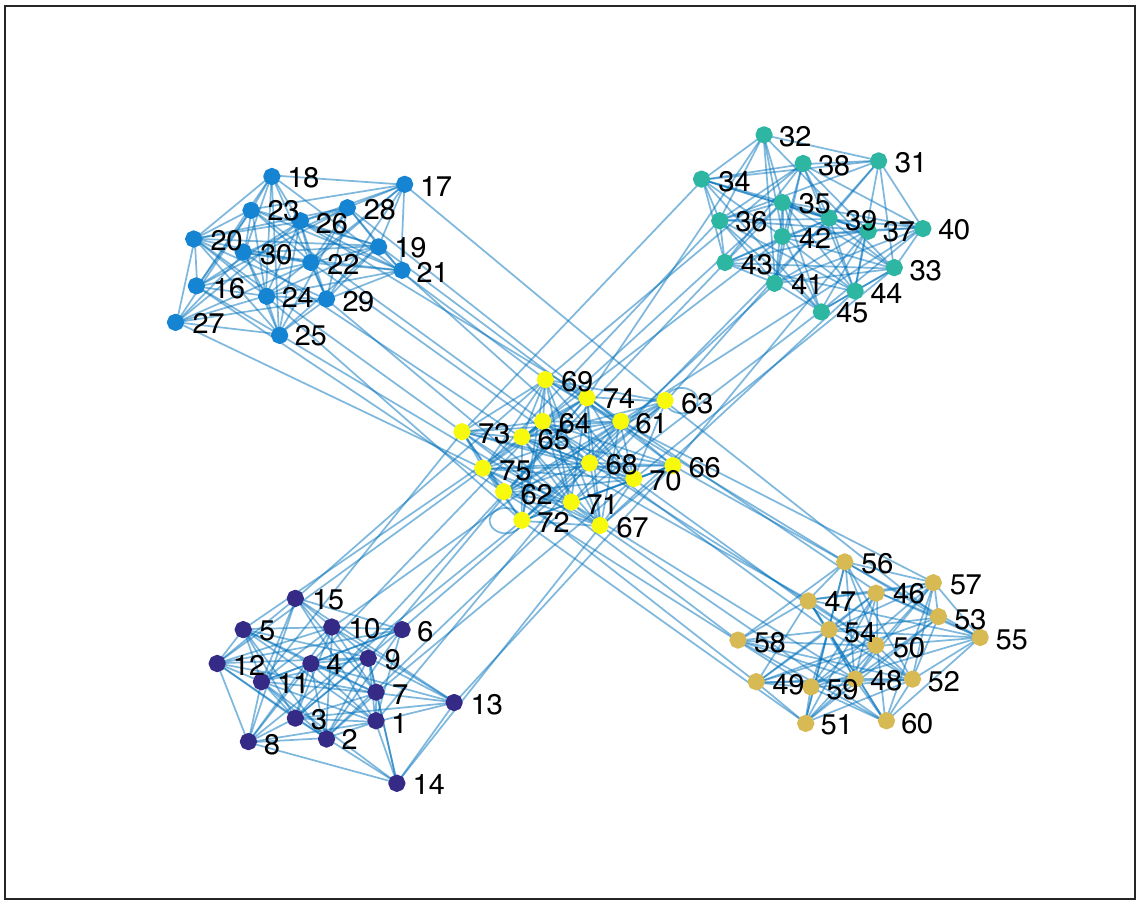}\label{figure_community_detection_5clusters}}
\hspace{5pt}
\subfigure[$\alpha = 1$, $187$ seconds]{\includegraphics[width=0.46\columnwidth]{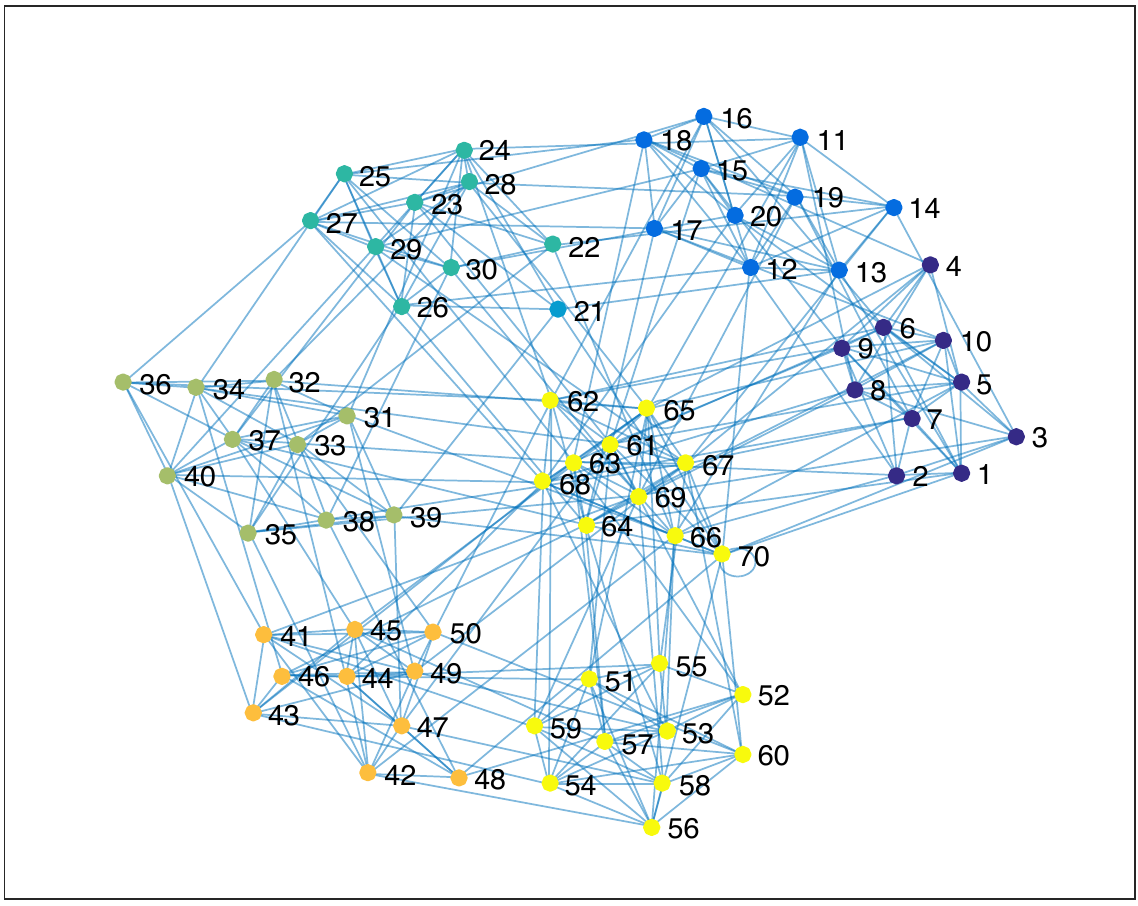}\label{figure_community_detection_7clusters}}
\hspace{5pt}
\subfigure[$\alpha=5$, $489$ seconds]{\includegraphics[width=0.46\columnwidth]{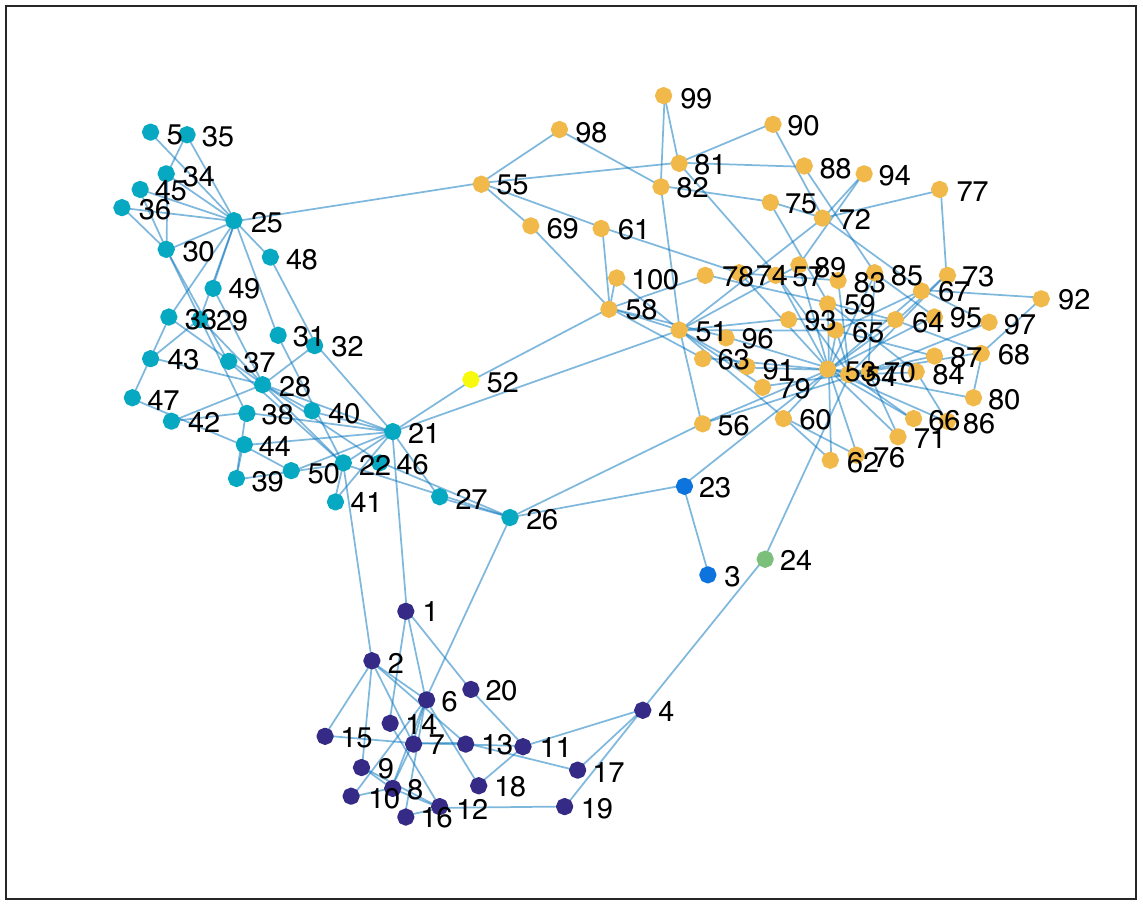}\label{figure_community_detection100-100}}
\hspace{5pt}
\subfigure[$\alpha = 170$, $50$ seconds]{\includegraphics[width=0.46\columnwidth]{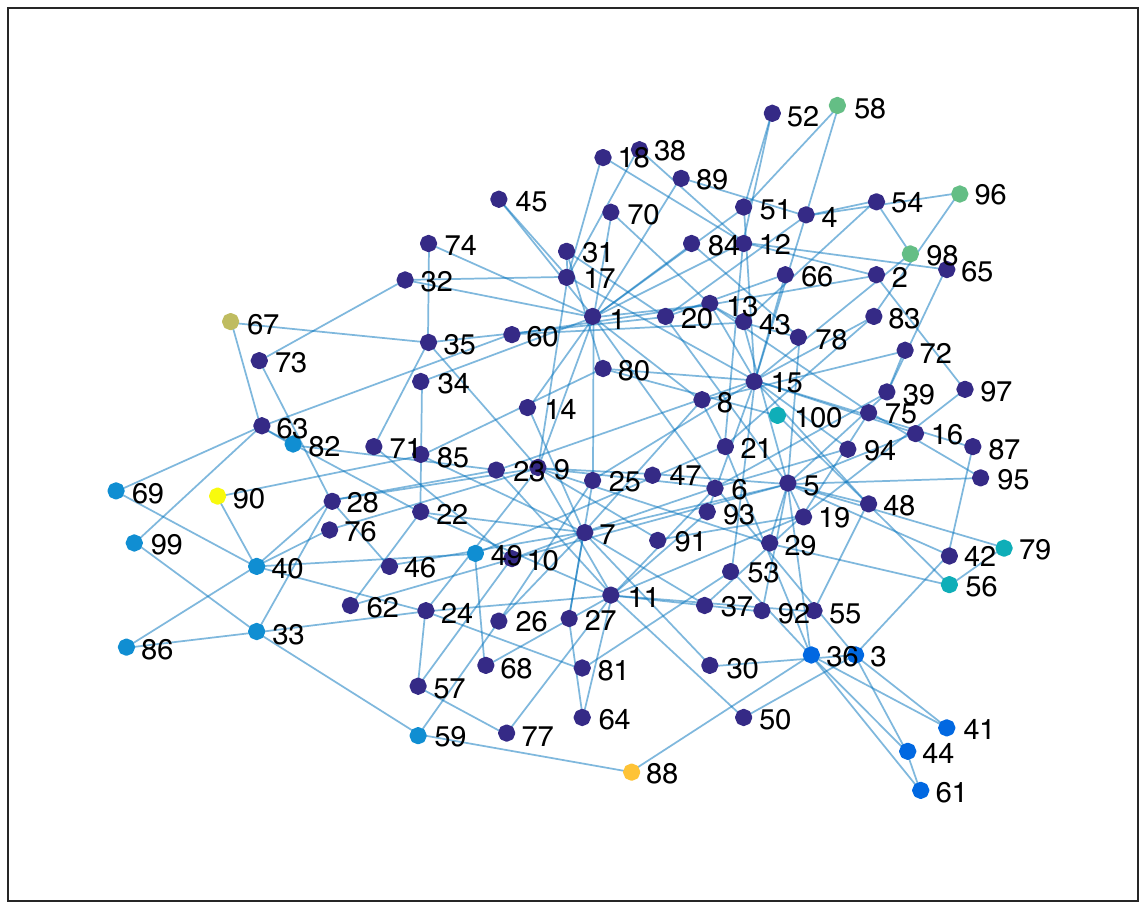}\label{figure_community_detection_100-8clusters}}
\caption{The community detection is conducted via the triangle lasso by using the ADMM method.}
\label{figure_community_detection}
\end{figure*}
\begin{figure}[t]
\centering 
\subfigure[No preturbation]{\includegraphics[width=0.49\columnwidth]{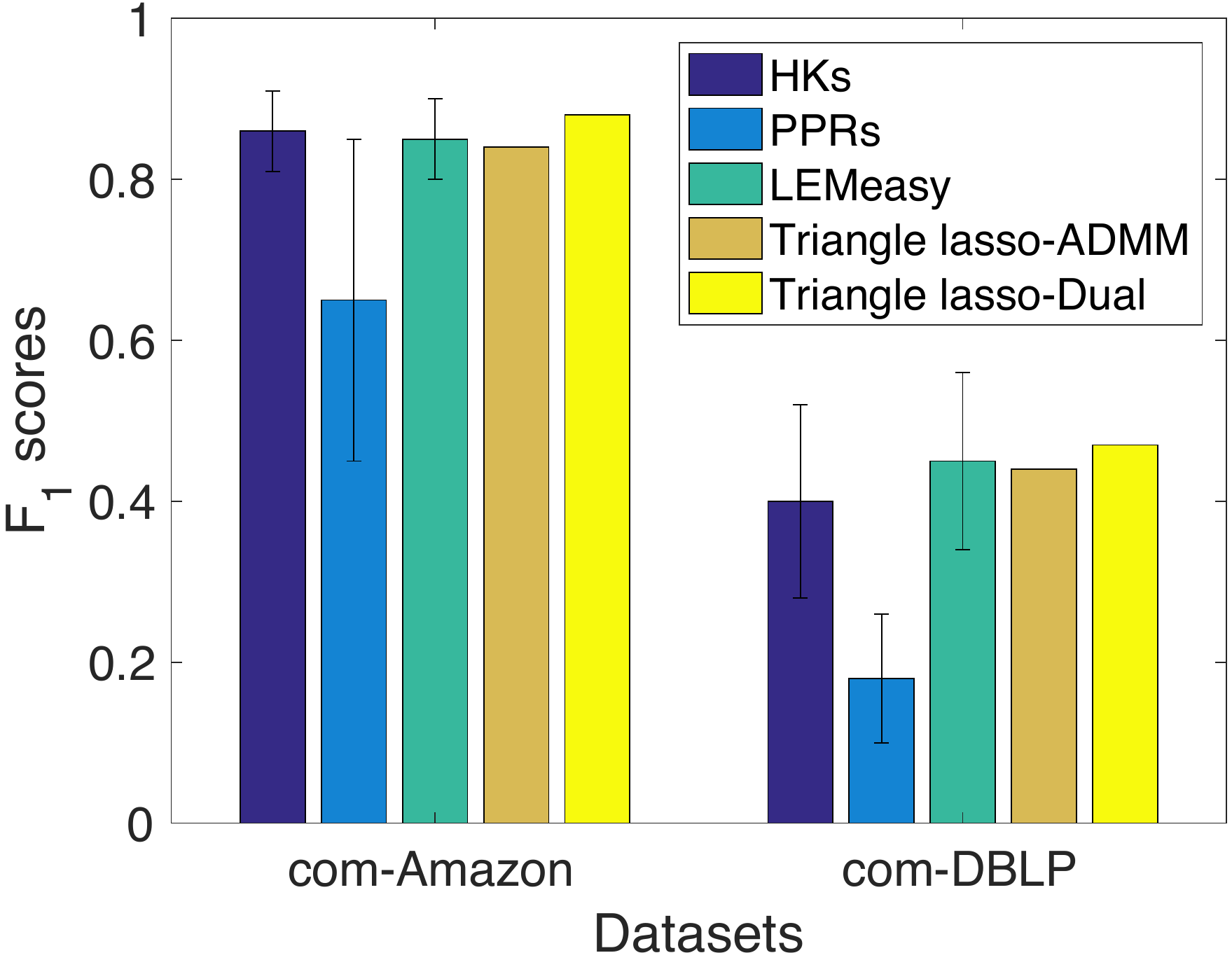}\label{figure_community_detection_f1_scores}}
\subfigure[Preturbation]{\includegraphics[width=0.49\columnwidth]{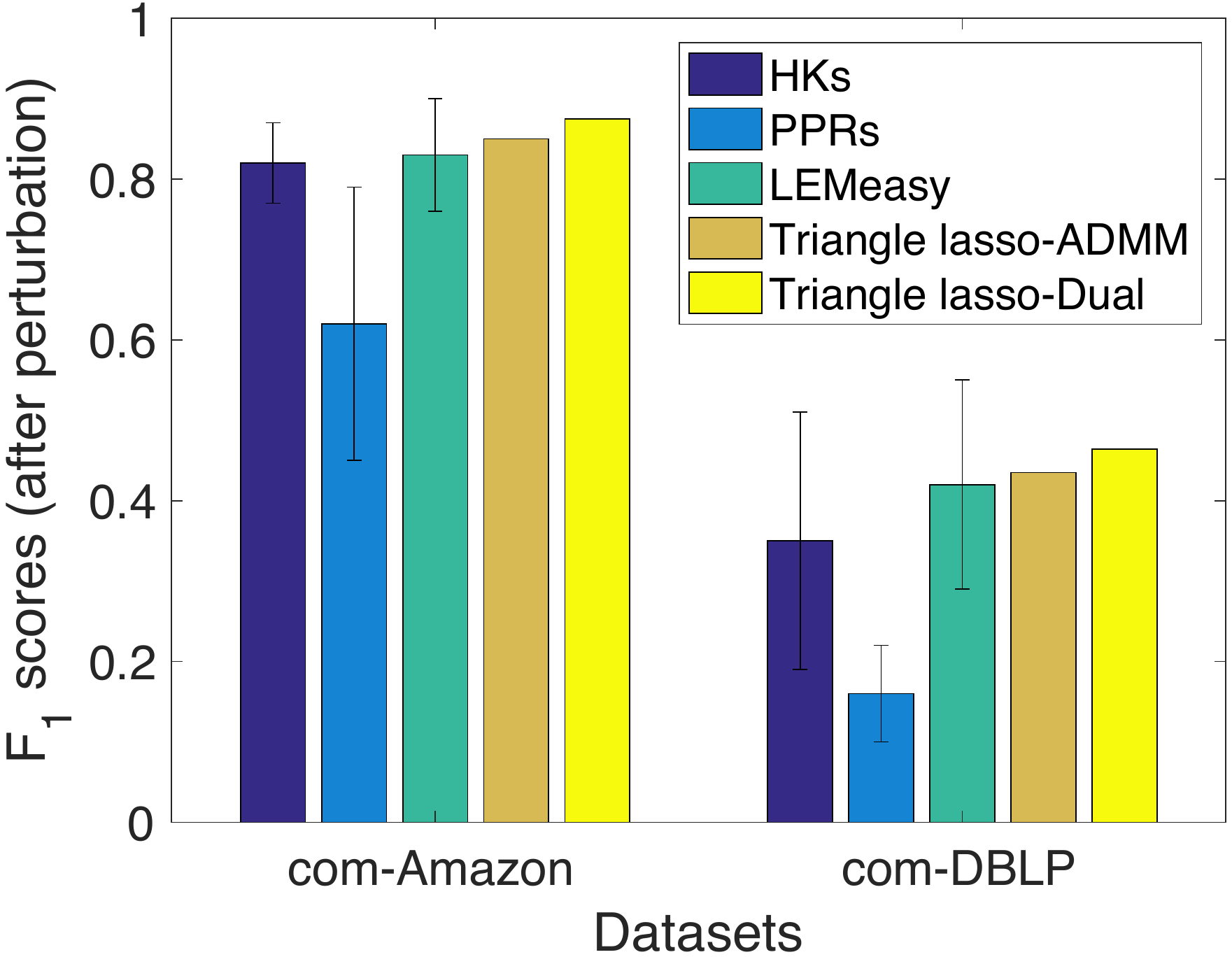}\label{figure_community_detection_f1_scores_perturb}}
\caption{Comparison of $F_1$ scores when detecting communities on the raw and the perturbed datasets.}
\label{figure_community_detection_f1}
\end{figure}

\subsection{Efficiency in various networks}
\textbf{Dataset and networks.} The dataset is yielded from a Gauss distribution whose mean is varied as $-5$, $-3$, $0$, $3$, and $5$, and its standard deviation is $1$. We still use ridge regression to test the efficiency of our methods. The Gauss distribution generates $20$  instances in each setting. The total number of instances in the dataset is $100$.  Each instance, e.g. $A_i$ is a $5$ dimensional row vector. Besides, the response for an instance, e.g. $y_i$ is set to be the mean of its Gauss distribution. $\gamma=0.01$, and $\alpha=0.01$. We evaluate the efficiency of our methods by varying the average degree in the classic networks: the random network, the small world network and the scale free network. Besides, we yield a network where there is a Completely Connected Community (C3) in the network.

\textbf{Results.} As illustrated in Fig. \ref{figure_effi_networks}, our ADMM is more efficient than the Dual method. The reason is that our Dual method has to solve a large number of linear equations, which is time-consuming. Furthermore, we can obtain the following observations. (1) Both our ADMM and Dual methods have the similar efficiency in the random and small world networks according to Fig. \ref{figure_evaluation_runtime_random_small_world}. (2) Our Dual method is performed very fast in the scale free network. It shows that our Dual method is suitable to solve a large-scale problem in the scale free network according to Fig. \ref{figure_evaluation_runtime_scale_free} and \ref{figure_evaluation_runtime_different_networks_admm_dual}. But, our ADMM method in the scale free network has the comparable efficiency in the random and small world networks. (3) Our Dual method has a relatively low scalability in the C3 network according to Fig. \ref{figure_evaluation_runtime_partial_dense} and \ref{figure_evaluation_runtime_different_networks_admm_dual}. If some vertices are completely connected in a network, the efficiency of our Dual method will be decreased sharply. The reason is that the weights of a vertex are highly non-separable with that of their neighbours in the completely connected community.  This fact illustrates that we should avoid to perform the  Dual method in such a network.

\subsection{Community detection}

\begin{table}[!]
\centering
\caption{Statistics of the network datasets with ground truth.}
\label{table_network_datasets}
\begin{tabular}{c|c|c|c}
\hline 
Datasets & \# Nodes & \# Edges & \# Communities  \tabularnewline
\hline 
\hline  
com-Amazon & $334,863$ & $925,872$ & $75,149$\tabularnewline
\hline 
com-DBLP & $317,080$ & $1,049,866$ & $13,477$\tabularnewline
\hline 
\end{tabular}
\end{table}

\textbf{Dataset and settings.} First, we synthetize four graph datasets, and present that the performance of triangle lasso for community detection.  
Second, some real graph datasets with ground truth are provided in the SNAP repository \cite{snapnets,Yang2012DefiningAE}. We use two of them: com-Amazon \footnote{https://snap.stanford.edu/data/com-Amazon.html} and com-DBLP \footnote{https://snap.stanford.edu/data/com-DBLP.html} to test the performance of triangle lasso quantitatively. The statics of those datasets are illustrated in Table \ref{table_network_datasets}.

Comparing with triangle lasso, we conduct the community detection by using the state-of-the-art methods: HKs \cite{Kloster:2016ue}, PPRs \cite{Kloster:2016ue} and LEMeasy \cite{Kloster:2016ue}, which are implemented in the open source project \cite{lemon-sqz_project}. Additionally, we use the $F_1$ score to test the performance of the community detection quantitatively. The quantitative metric is the $F_1$ score, which is defined by
\begin{align}
\nonumber
F_1 = 2\cdot \frac{\mathrm{precision}\times \mathrm{recall}}{\mathrm{precision} + \mathrm{recall}}.
\end{align} 

More details about the metric are recommended to refer to \cite{f1_score_definition}.  Note that the previous methods HKs, PPRs and LEMeasy need to set a seed for each a community. Since we know the ground truth of the communities, we randomly pick a member of a community, and then set it as the seed for the community. Thus, we obtain an $F_1$ for the community. Then, the average of those $F_1$ scores for all communities is used to represent the $F_1$ score based on the seed. Repeating the procedure $10$ times, we record the mean and variance of those averaged $F_1$ scores to evaluate the performance of all the methods quantitatively.

\textbf{Results.} As illustrated in Fig. \ref{figure_community_detection}, the triangle lasso is able to detect multiple communities. Fig. \ref{figure_community_detection_5clusters}, \ref{figure_community_detection_7clusters}, and \ref{figure_community_detection100-100} show that it performs very well when the network consists of multiple communities. Fig. \ref{figure_community_detection_100-8clusters} shows that the triangle lasso still finds the local community structure of the network when the community structure is not obvious.

Quantitatively, as illustrated in Fig. \ref{figure_community_detection_f1_scores}, our dual method obtains the highest $F_1$ scores, outperforming the state-of-the-art methods significantly.  It is highlighted that both our methods including the dual method and the ADMM method yield the deterministic solutions, which are not impacted by the initial values to start them. Thus, the variance of the solutions yielded by our methods is $0$. But, the previous methods are heuristic, and their solutions are sensitive to the seeds which are selected before running them.

Furthermore, we evaluate the robustness of the solution by performing those methods on the perturbed datasets. The perturbation includes the following steps. 
\begin{itemize}
\item \textit{We randomly select a node which is a member of a community, and a node which is a non-member of the community.}
\item \textit{An edge is generated to connect those two nodes.}
\end{itemize} 

The number of those member nodes is controlled to be less than $1\%$ of the total number of the nodes in the selected community. Fig. \ref{figure_community_detection_f1_scores_perturb} shows that both our ADMM method and our dual method outperform their counterparts under the perturbation strategy.  The reason is that triangle lasso uses the neighbouring information to yield a robust solution, and thus decreases the impact of the perturbation.

\section{Conclusion}
\label{sect_conclusion}
It is challenging to simultaneously clustering and simultaneously in practical datasets due to the imperfect data.   In the paper, we formulate the triangle lasso as a convex problem. After that, we develop the ADMM method to obtain the moderately accurate solution. Additionally,  we transform the original problem to be a second-order cone programming problem, and solve it in the dual space.  Finally, we conduct extensive empirical studies to show the superiorities  of triangle lasso.


%

\section*{Acknowledgment}
This work was supported by the National Key R \& D Program of China 2018YFB1003203 and the National Natural Science Foundation of China (Grant No. 61672528, 61773392, 60970034, 61170287, 61232016, and 61671463). We thank Cixing intelligent manufacturing research institute, Cixing textile automation research institute, Ningbo Cixing corporation limited and Ningbo Cixing robotics company limited because of their financial support and application scenarios.




%


\bibliographystyle{IEEEtran}  
\bibliography{reference}

%



\begin{IEEEbiography}[{\includegraphics[width=1in,height=1.25in,clip,keepaspectratio]{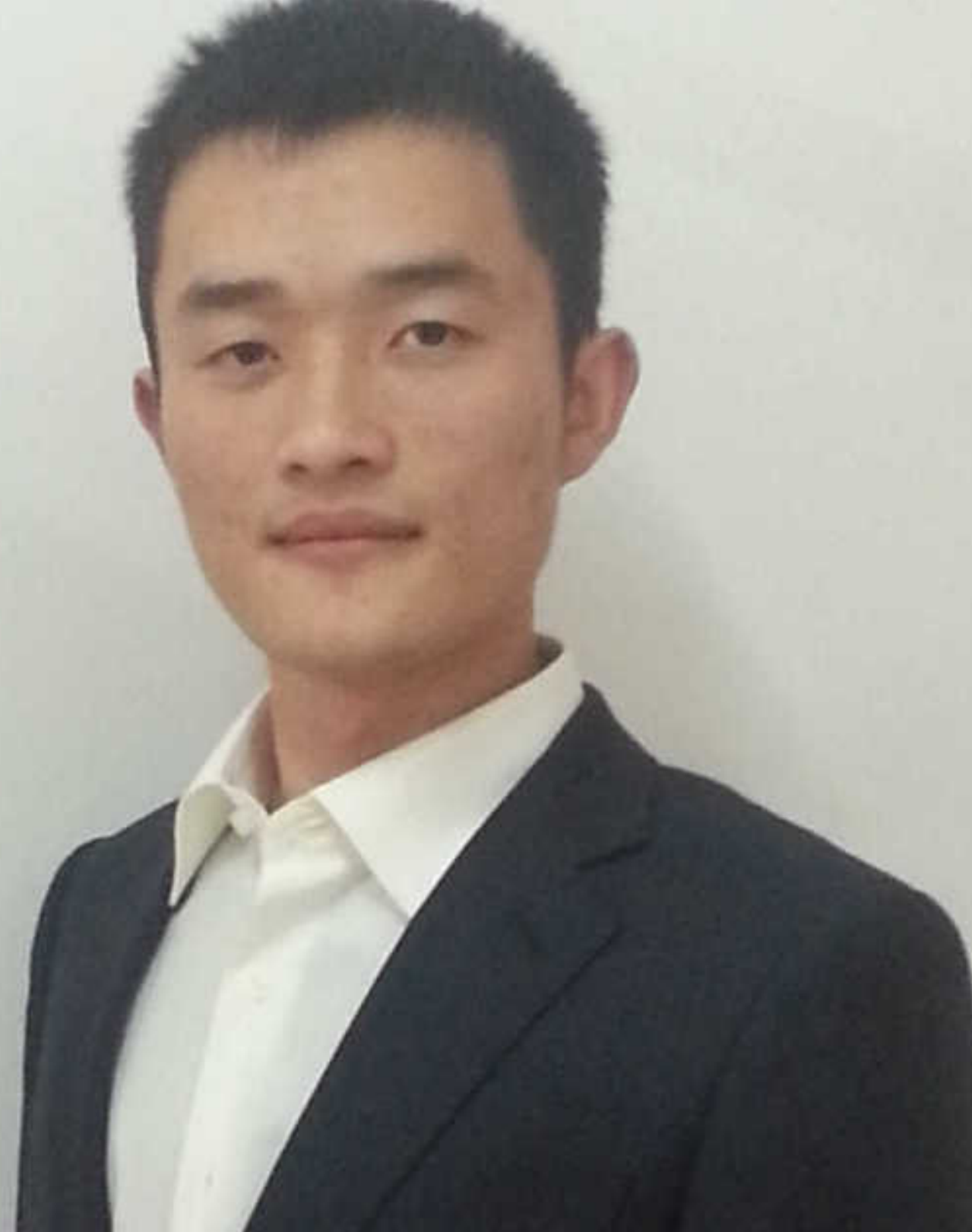}}]{Yawei Zhao} is currently a Ph.D. candidate in Computer Science from the National University of Defense Technology, China. He received his B.E. degree and M.S. degree in Computer Science from the National University of Defense Technology, China, in 2013 and 2015, respectively. His research interests include asynchronous and parallel optimization algorithms, pattern recognition and machine learning.
\end{IEEEbiography}

\begin{IEEEbiography}[{\includegraphics[width=1in,height=1.25in,clip,keepaspectratio]{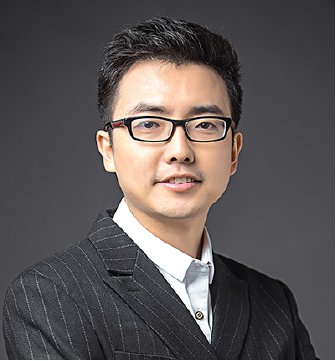}}]{Kai Xu} is an Associate Professor at the School of Computer, National University of Defense Technology, where he received his Ph.D. in 2011. From 2008 to 2010, he conducted visiting research at Simon Fraser University. He is visiting Princeton University since July 2017. His research interests include geometry processing and geometric modeling, especially on data-driven approaches to the problems in those directions, as well as 3D-gemoetry-based computer vision for robotic applications. 
\end{IEEEbiography}

\begin{IEEEbiography}[{\includegraphics[width=1in,height=1.25in,clip,keepaspectratio]{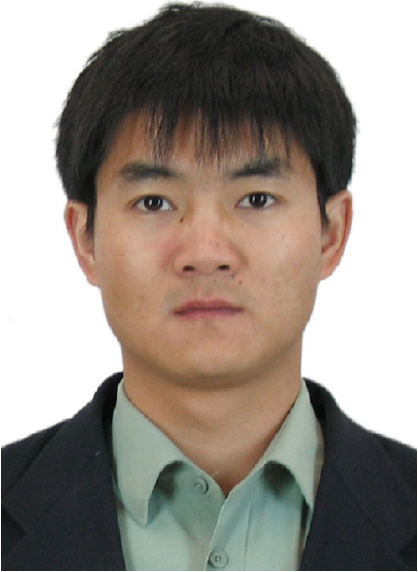}}]{Xinwang Liu} received his PhD degree from National University of Defense Technology (NUDT), China. He is now Assistant Researcher of School of Computer Science, NUDT. His current research interests include kernel learning and unsupervised feature learning. Dr. Liu has published 40+ peer-reviewed papers, including those in highly regarded journals and conferences such as IEEE T-IP, IEEE T-NNLS, ICCV, AAAI, IJCAI, etc. He served on the Technical Program Committees of IJCAI 2016-2017, AAAI 2016-2018.
\end{IEEEbiography}

\begin{IEEEbiography}[{\includegraphics[width=1in,height=1.25in,clip,keepaspectratio]{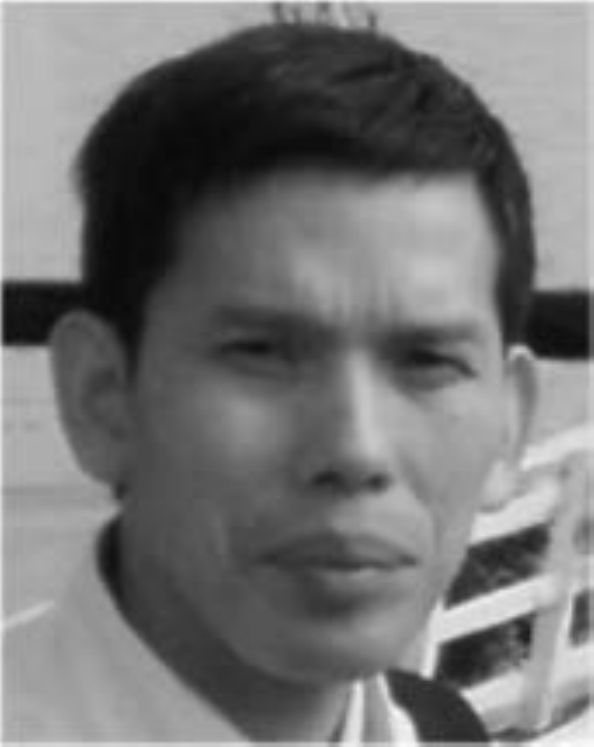}}]{En Zhu} received his M.S. degree and Ph.D. degree in Computer Science from the National University of Defense Technology, China, in 2001 and 2005, respectively. He is now working as a full professor in the School of Computer Science, National University of Defense Technology, China. His main research interests include pattern recognition, image processing, and information security.
\end{IEEEbiography}

\begin{IEEEbiography}[{\includegraphics[width=1in,height=1.25in,clip,keepaspectratio]{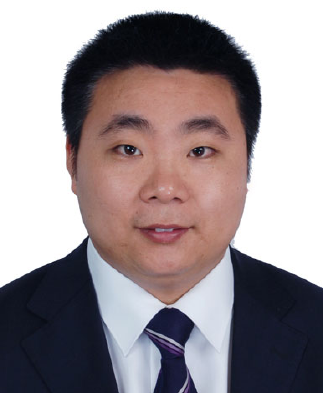}}]{Xinzhong Zhu} is a professor at College of Mathematics, Physics and Information Engineering, Zhejiang Normal University, PR China. He received his B.S. degree in computer science and technology from the University of Science and Technology Beijing (USTB), Beijing, in 1998, and received M.S. degree in software engineering from National University of Defense Technology in 2005, and now he is a Ph.D. candidate at XIDIAN University.  His research interests include epileptic seizure detection, machine learning, computer vision, manufacturing informatization, robotics and system integration, and intelligent manufacturing. He is a member of the ACM.
\end{IEEEbiography}

\begin{IEEEbiography}[{\includegraphics[width=1in,height=1.25in,clip,keepaspectratio]{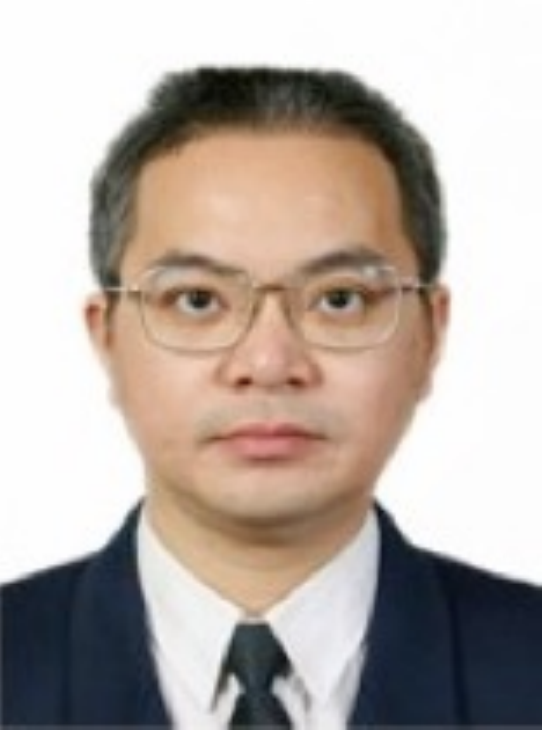}}]{Jianping Yin} received his M.S. degree and Ph.D. degree in Computer Science from the National University of Defense Technology, China, in 1986 and 1990, respectively. He is a professor of computer science in the Dongguan University of Technology. His research interests involve artificial intelligence, pattern recognition, algorithm design, and information security.
\end{IEEEbiography}

\end{document}